\documentclass{article}

\usepackage[preprint]{neurips_2022}

\usepackage[utf8]{inputenc} 
\usepackage[T1]{fontenc}    
\usepackage{hyperref}       
\usepackage{url}            
\usepackage{booktabs}       
\usepackage{amsfonts}       
\usepackage{nicefrac}       
\usepackage{microtype}      
\usepackage{xcolor}         
\usepackage{amsmath}
\usepackage{amssymb}
\usepackage{amsfonts}
\usepackage{bbm}
\usepackage[linesnumbered,ruled]{algorithm2e}
\newcommand{\vect}[1]{\boldsymbol{\mathbf{#1}}}
\usepackage{xcolor}
\usepackage{newfloat}
\usepackage{comment}
\usepackage{listings}
\usepackage{natbib}
\newtheorem{corollary}{\textbf{Corollary}}
\newtheorem{theorem}{\textbf{Theorem}}

\newtheorem{proof}{\textit{Proof}}
\usepackage{enumitem}
\newtheorem{lemma}{\textbf{Lemma}}
\usepackage[pdftex]{graphicx}
\usepackage{lipsum}
\usepackage[caption=false,font=footnotesize]{subfig}
\usepackage{wrapfig}
\newcommand{\PB}[1]{{\color{red} Prasanna: {#1}}}

\title{Multifidelity Reinforcement Learning with \\ Control Variates}

\author{%
  Sami Khairy
  \\
  Argonne National Laboratory\\
  Lemont, IL 60439 \\
  \texttt{ skhairy@anl.gov} \\
  \And
  Prasanna Balaprakash \\
  Argonne National Laboratory\\
  Lemont, IL 60439 \\
  \texttt{pbalapra@anl.gov} \\
}

\begin{document}

\maketitle

\begin{abstract}

In many computational science and engineering applications, the output of a system of interest corresponding to a given input can be queried at different levels of fidelity with different costs. Typically, low-fidelity data is cheap and abundant, while high-fidelity data is expensive and scarce. In this work we study the reinforcement learning (RL) problem in the presence of multiple environments with different levels of fidelity for a given control task. We focus on improving the RL agent's performance with multifidelity data. Specifically, a multifidelity estimator that exploits the cross-correlations between the low- and high-fidelity returns is proposed to reduce the variance in the estimation of the state-action value function. The proposed estimator, which is based on the method of control variates, is used to design a multifidelity Monte Carlo RL \texttt{(MFMCRL)} algorithm that improves the learning of the agent in the high-fidelity environment. The impacts of variance reduction on policy evaluation and policy improvement are theoretically analyzed by using probability bounds. Our theoretical analysis and numerical experiments demonstrate that for a finite budget of high-fidelity data samples, our proposed \texttt{MFMCRL} agent attains  superior performance compared with that of a standard RL agent that uses only  the high-fidelity environment data for learning the optimal policy.
\end{abstract}

\section{Introduction}

Within the computational science and engineering (CSE) community, multifidelity data refers to data that comes from different sources with different levels of fidelity. The criteria by which data is considered to be low fidelity or high fidelity vary across different applications, but  usually  low-fidelity data is much cheaper to generate than high-fidelity data under some cost metric. In robotics for instance, data coming from a robot operating in the real world constitutes high-fidelity data, while simulated data of the robot based on first principles is considered to be low-fidelity data. Different simulators of the robot can also be designed by increasing the modeling complexity. A simulator that takes into account aerodynamic drag is, for instance, of higher fidelity than one that is based only on the simple laws of motion. As another example, a neural classifier in deep learning can be trained on the \textit{full} training data for a \textit{large} number of training epochs, or on a \textit{subset} of the training data for \textit{few} epochs. Evaluating the trained model on a held-out validation data set in the former case  yields a higher-fidelity estimate of the classifiers' performance compared with that in the latter case. In general, 
low-fidelity data serves as an approximation to its high-fidelity counterpart and can be generated cheaply and abundantly \cite{meng2020composite}. Many outer-loop applications that require querying the system at many different inputs, including black-box optimization \cite{li2020multi}, inference \cite{perdikaris2017nonlinear}, and uncertainty propagation \cite{koutsourelakis2009accurate,peherstorfer2016optimal}, can exploit the cross-correlations between low- and  high-fidelity data to solve new problems that would otherwise be prohibitively costly to solve using high-fidelity data alone \cite{peherstorfer2018survey,perdikaris2017nonlinear}. 


Motivated by the advent of multifidelity data sources within CSE, in this work we study the reinforcement learning (RL) problem in the presence of multiple environments 
with different levels of fidelity for a given control task. RL is a popular machine learning paradigm for intelligent sequential decision-making under uncertainty, enabling data-driven control of complex systems with scales ranging from quantum \cite{khairy2020learning} to cosmological \cite{moster2021galaxynet}. State-of-the-art model-free RL algorithms have indeed demonstrated sheer success for learning complex policies from raw data in single-fidelity environments \cite{mnih2015human, lillicrap2015continuous,schulman2015trust,schulman2017proximal,pmlr-v80-haarnoja18b}. This success, however, comes at the cost of requiring a large number of data samples to solve a control task \textit{satisfactorily}.\footnote{Poor sample complexity of model-free RL algorithms has long motivated developments in model-based RL, where a predictive model of the environment is learned alongside the policy \cite{janner2019trust, schrittwieser2020mastering}. Our work is focused on model-free RL.} In the presence of multiple environments with different levels of fidelity, new ways arise that could help the agent learn better policies. One way that has been well studied in the context of RL is \textit{transfer learning (TL)}. In TL \cite{taylor2009transfer,chebotar2019closing,zhu2020transfer}, the agent first uses the low-fidelity environment to learn a policy that is then transferred (directly or indirectly through the transfer of the state-action value function) to the high-fidelity environment as a heuristic to bootstrap learning. Essentially, TL attempts to leverage multifidelity environments to deal with
the exploration-exploitation dilemma that is present within RL, and it works under the assumption that the maximum deviation between the optimal low-fidelity state-action value function and the optimal high-fidelity state-action value function is bounded with a threshold that is used by TL for bootsrapping the high-fidelity value function \cite{cutler2015real}. In our work we explore an uncharted territory and focus on \textit{multifidelity} estimation in RL and its role in improving the learning of the agent. We demonstrate that as long as the low- and high-fidelity state-action value functions for any policy are correlated, significant performance improvements can be reaped by leveraging these cross-correlations without extra effort in managing the exploration-exploitation process.

The main contributions of our work are summarized as follows. First, we study a generic multifidelity setup in which the RL agent can execute a policy in two environments, a low-fidelity environment and a high-fidelity environment. To leverage the cross-correlations between the low- and high-fidelity returns, we propose an unbiased reduced-variance multifidelity estimator for the state-action value function based on the framework of control variates. Second, a multifidelity Monte Carlo (MC) RL algorithm, named \texttt{MFMCRL}, is proposed to improve the learning of the RL agent in the high-fidelity environment. For any finite budget of high-fidelity environment interactions, \texttt{MFMCRL} leverages low-fidelity data to learn better policies than a standard RL agent  that  uses only  the high-fidelity data. Third, we theoretically analyze the impacts of variance reduction in the estimation of the state-action value function on policy evaluation and policy improvement using probability bounds. Fourth, performance gains of the proposed \texttt{MFMCRL} algorithm are empirically assessed through numerical experiments in synthetic multifidelity environments, as well as a neural architecture search (NAS) use case. 

\section{Preliminaries and related work}

\subsection{Reinforcement learning}
We consider episodic RL problems where the environment $\Sigma$ is specified by an infinite-horizon Markov decision process (MDP) with discounted returns \cite{bellman1957markovian}. Specifically, an infinite-horizon MDP is defined as a tuple $\mathcal{M}=(\mathcal{S},\mathcal{A},\mathcal{P}, \vect{\beta},$ $\mathcal{R},\gamma)$,  where $\mathcal{S}$ and $\mathcal{A}$ are finite sets of states and actions, respectively; $\mathcal{P}:\mathcal{S}\times \mathcal{A} \times \mathcal{S} \rightarrow [0,1]$ is  the environment dynamics; and $\vect{\beta}:\mathcal{S} \rightarrow [0,1]$ is the initial distribution over the states, that is, $\beta(s) = \text{Pr}(s_0=s), \forall s \in \mathcal{S}$. The reward function $\mathcal{R}$ is  bounded and defined as $\mathcal{R}: \mathcal{S} \times \mathcal{A} \rightarrow [R_\text{min}, R_\text{max}]$, where $R_\text{min}$ and  $R_\text{max}$ are real numbers. $\gamma$ is a discount factor to bound the cumulative rewards and trade off how far- or short-sighted the agent is in its decision making. The environment dynamics, $\mathcal{P}(s^\prime | s,a), \forall s,a,s^\prime \in \mathcal{S}\times \mathcal{A} \times \mathcal{S}$, encode the stationary transition probability from a state $s $ to a state $s^\prime$ given that action $a$ is chosen \cite{bertsekas2000dynamic,kallenberg2011markov}. In the episodic setting, there exists at least one terminal state $s_T$ such that $\mathcal{P}(s^\prime|s_T, a)=0, \forall a, s^\prime \neq s_T$ and $\mathcal{P}(s_T|s_T, a)=1, \forall a$, i.e. $s_T$ is an absorbing state. Furthermore, $\beta(s_T)=0$ and $\mathcal{R}(s_T,a)=0, \forall a$. When the RL agent transitions into a terminal state, all subsequent rewards are zero, and simulation is restarted from another state $s \sim \beta$.

The agent's decision-making process is characterized by $\pi(a|s)$, which is a Markov stationary policy that defines a distribution over the actions $a \in \mathcal{A}$ given a state $s \in \mathcal{S}$. In the RL problem, $\mathcal{P}$ and $\mathcal{R}$ are not known to the agent, yet the agent can interact with the environment sequentially at discrete time steps, $t=0,1,2,\cdots, T$,  by exchanging actions and rewards. Notice that $T$ is a random variable and denotes the time step at which the agent transitions into a terminal state. At each time step $t$, the agent observes the environment's state $s_t=s \in \mathcal{S}$, takes  action $a_t=a \sim \pi(a|s) \in \mathcal{A}$,  and receives a reward $r_{t+1} = \mathcal{R}(s,a)$. The environment's state then evolves to a new state $s_{t+1}=s^\prime \sim \mathcal{P}(s^\prime|s,a)$. The state-value function of a state $s$ under a policy $\pi$ is defined as the expected long-term discounted returns starting in state $s$ and following policy $\pi$ thereafter, $V_\pi(s) = \mathbb{E}_{a_t\sim \pi, s_t \sim \mathcal{P}}\bigg[\sum_{t=0}^\infty \gamma^t \mathcal{R}(s_t,a_t)|s_0 = s \bigg]$. 
In addition, the state-action value function of a state $s$ and action $a$ under a policy $\pi$ is defined as $Q_\pi(s,a) = \mathbb{E}_{a_t\sim \pi, s_t \sim \mathcal{P}}\bigg[\sum_{t=0}^\infty \gamma^t \mathcal{R}(s_t,a_t)|s_0 = s, a_0 = a \bigg]$. 
Notice that $V_\pi(s) = \mathbb{E}_{a\sim \pi}[Q_\pi(s,a)]$. 
The solution of the RL problem is a policy $\pi^*$ that maximizes the discounted returns from the initial state distribution $
\pi^* = \underset{\pi}{\text{argmax}}~~ \mathbb{E}_{s \sim \beta}[V_{\pi}(s)]$. 
It is well known that there exists at least one optimal policy $\pi^*$ such that $V_{\pi^*}(s) = \underset{\pi}{\text{max}}~ V_{\pi}(s), \forall s \in \mathcal{S}$ and $Q_{\pi^*}(s,a) = \underset{\pi}{\text{max}} ~Q_{\pi}(s,a), \forall s,a \in \mathcal{S \times A}$ \cite{altman1999constrained}. Furthermore, a deterministic policy that selects the greedy action with respect to $Q_{\pi^*}(s,a), \forall s \in \mathcal{S}$, is an optimal policy. 

\subsection{Control variates}
\label{Sec:CV}
The method of control variates is a variance reduction technique that leverages the correlation between random variables (r.vs.) to reduce the variance of an estimator \cite{lemieux2014control}. Let ${W}_1, {W}_2, \cdots, {W}_n$ be $n$ independent and identically distributed (i.i.d.) r.vs. such that $\mathbb{E}[W_i]=\mu_{_W}$, and $\mathbb{E}[(W_i-\mu_{_W})^2]=\sigma^2_{_W}, \forall i \in [n]$. In addition, let ${Z}_1, {Z}_2, \cdots, {Z}_n$ be $n$ i.i.d. r.vs. such that $\mathbb{E}[Z_i]=\mu_{_Z}$, and $\mathbb{E}[(Z_i-\mu_{_Z})^2]=\sigma^2_{_Z}, \forall i \in [n]$. Suppose that $W_i, Z_i$ are correlated with a correlation coefficient $\rho_{_{W,Z}} = \frac{\text{Cov}[Z_i,W_i]}{\sqrt{\sigma^2_{_Z}}\sqrt{ \sigma^2_{_W}}}, \forall i \in [n]$, where $\text{Cov}[Z_i,W_i]=\mathbb{E}[Z_iW_i]-\mathbb{E}[Z_i]\mathbb{E}[W_i]$ is the covariance between $Z_i$ and $W_i$. Furthermore, suppose that $W_i, Z_j$ are independent and thus uncorrelated $\forall i \neq j$. Using the Cauchy–-Schwartz inequality, one can  show  that $|\rho_{_{W,Z}}| \leq 1$. 

To estimate $\mu_{_W}$, we first consider the sample mean estimator, $\hat{\theta}_1 = \frac{1}{n} \sum_{i=1}^n W_i$. $\hat{\theta}_1$ is an unbiased estimator of $\mu_{_W}$, in other words, $\mathbb{E}[\hat{\theta}_1] = \frac{1}{n} \sum_{i=1}^n \mathbb{E}[W_i] = \mu_{_W}$, and has a variance $\text{Var}[\hat{\theta}_1] = \frac{\sigma^2_{_W}}{n}$. Next, we consider the control-variate-based estimator, 
\begin{equation} \label{eq:cv}
\hat{\theta}_2 = \frac{1}{n} \sum_{i=1}^n W_i + \alpha(Z_i - \mu_{_Z}).
\end{equation}
$\hat{\theta}_2$ is also an unbiased estimator of $\mu_{_W}$, i.e., $\mathbb{E}[\hat{\theta}_2] = \mu_{_W}$, yet it has a variance $\text{Var}[\hat{\theta}_2] = \frac{1}{n} \text{Var}[W_i + \alpha(Z_i -\mu_{_Z})]=\frac{1}{n} \big( \text{Var}[W_i] + \alpha^2 \text{Var}[Z_i] + 2\alpha \text{Cov}[Z_i,W_i] \big)$.  The variance of $\hat{\theta}_2$ can be controlled and minimized by setting $\alpha$ to the minima of $\text{Var}[W_i] + \alpha^2 \text{Var}[Z_i] + 2\alpha \text{Cov}[Z_i,W_i]$, which is attained at $\alpha^* = -\frac{\text{Cov}[Z_i,W_i]}{\sigma^2_{_Z}}= -\rho_{_{Z,W}}\frac{\sigma_{_W}}{\sigma_{_Z}}$. Hence,
by introducing $\alpha(Z_i - \mu_{_Z})$ as a control variate, the variance of $\hat{\theta}_2$ is reduced, \begin{equation}
    \text{Var}[\hat{\theta}_2] = (1-\rho^2_{_{Z,W}}) \text{Var}[\hat{\theta}_1].
\end{equation}
Because $\hat{\theta}_2$ is an unbiased estimator, $\hat{\theta}_2$ has a lower mean squared error (MSE) by the bias-variance decomposition theorem of the MSE. Applications of the method of control variates extend beyond variance reduction. For example, the concept of control variates is used in \cite{peherstorfer2016optimal} to design a fusion framework to combine an arbitrary number of surrogate models optimally.

\subsection{Related work}


In \cite{abbeel2006using},  a policy search algorithm is proposed that leverages a crude approximate model $\hat{\mathcal{P}}$ of the true MDP to quickly learn to perform well on real systems. The proposed algorithm, however, is limited to the case where ${\mathcal{P}}$ is deterministic, and it assumes that model derivatives are good approximations of the true derivatives such that policy gradients can be computed by using the approximate model. 
In transfer learning (TL)  \cite{taylor2007transfer,mann2013directed}, value, model, or policy parameters are transferred in one direction as a heuristic initialization to bootstrap learning in the high-fidelity environment, with no option for backtracking. The option for the agent to backtrack and to choose which environment to use is studied in the multifidelity RL (MFRL) work of \cite{cutler2015real}. That algorithm is extended  in \cite{suryan2020multifidelity}  by integrating function approximation using Gaussian processes \cite{williams2006gaussian}. As in TL, both \cite{cutler2015real} and \cite{suryan2020multifidelity} use the value function from a lower-fidelity environment as a heuristic to bootstrap learning and \textit{guide exploration} in the high-fidelity environment. From an optimization viewpoint, 
this approach is reasonable only if the lower-fidelity  value  function lies in the vicinity of the optimal high-fidelity value function, a  situation that cannot be guaranteed or known a priori in general.
Hence, in \cite{cutler2015real,suryan2020multifidelity}, it is assumed that the optimal state-action value function in the low- and high-fidelity environments differ by no more than a small parameter $\beta$ at every state-action pair, and they require the knowledge of $\beta$ a priori to manage exploration-exploitation across multifidelity environments. By contrast, we  require only  that the low- and high-fidelity returns are correlated in our work, and the correlation need not be known a priori. The cross-correlation between the low- and high-fidelity returns is used for reducing the variance in the \textit{estimation} of the high-fidelity state-action value function, and hence our approach is complementary to existing TL techniques that use multifidelity environments for guided exploration \cite{cutler2015real,suryan2020multifidelity}. We show that as long as the low- and high-fidelity state-action value function of a policy are correlated, the agent can benefit from the cheap and abundantly available low-fidelity data to improve its performance, without altering the exploration process.


\section{Multifidelity estimation in RL}
\subsection{Problem setup}
We consider a multifidelity setup in which the RL agent has access to two environments, $\Sigma^{\text{lo}}$ and $\Sigma^{\text{hi}}$, modeled by the two MDPs $\mathcal{M}^{\text{lo}} = (\mathcal{S}^{\text{lo}},\mathcal{A},\mathcal{P}^{\text{lo}}, \vect{\beta}^{\text{lo}},\mathcal{R}^{\text{lo}},\gamma)$, and $\mathcal{M}^{\text{hi}} = (\mathcal{S}^{\text{hi}},\mathcal{A},\mathcal{P}^{\text{hi}}, \vect{\beta}^{\text{hi}},\mathcal{R}^{\text{hi}},\gamma)$, respectively, as shown in Figure \ref{fig:mfrl}.
$\Sigma^{\text{lo}}$ is a low-fidelity environment in which the low-fidelity reward function  $\mathcal{R}^{\text{lo}}: \mathcal{S \times A} \rightarrow [R_\text{min}^\text{lo},R_\text{max}^\text{lo}]$ and the low-fidelity dynamics $\mathcal{P}^{\text{lo}}$ are cheap\footnote{Sampling cost is application dependent. It is up to the practitioner to assign cost and determine low- and high-fidelity sampling budgets.} to evaluate/simulate, yet they are potentially inaccurate. On the other hand, $\Sigma^{\text{hi}}$ is a high-fidelity environment in which the high-fidelity reward function $\mathcal{R}^{\text{hi}}: \mathcal{S \times A} \rightarrow [R_\text{min}^\text{hi},R_\text{max}^\text{hi}]$ and the high-fidelity dynamics $\mathcal{P}^{\text{hi}}$ describe the real-world system with the highest accuracy, yet they are expensive to evaluate/simulate \cite{fernandez2016review}. We stress that $(\mathcal{P}^{\text{hi}}, \vect{\beta}^{\text{hi}},\mathcal{R}^{\text{hi}})$ and $(\mathcal{P}^{\text{lo}}, \vect{\beta}^{\text{lo}},\mathcal{R}^{\text{lo}})$ are \textbf{unknown} to the agent, and interaction with the two environments is only through the exchange of states, actions, next states and rewards, which is the typical case in RL.

The action space $\mathcal{A}$ is the same in both environments, yet the state space may differ. It is assumed that the low-fidelity state space is a subset of the high-fidelity state space, $\mathcal{S}^{\text{lo}} \subseteq \mathcal{S}^{\text{hi}}$, in  other words, the states  available  in the low-fidelity environment are a subset of those available at the high-fidelity environment, and it is assumed that there exists a known mapping\footnote{$\mathcal{T}$ is problem-specific. For instance, if $\mathcal{S}^{\text{hi}}$ represents a fine grid and $\mathcal{S}^{\text{lo}}$ represents a coarse grid, then $\mathcal{T}$ will map $s^{\text{hi}}$ to the closest $s^{\text{lo}}$ based on a chosen distance metric.} $\mathcal{T}: \mathcal{S}^{\text{hi}} \rightarrow \mathcal{S}^{\text{lo}}$ as in previous works \cite{taylor2007transfer,cutler2015real}. High-fidelity environments usually capture more state information than do low- fidelity environments so $\mathcal{T}$ can be a many-to-one map. 
Access to the high-fidelity simulator $\Sigma^{\text{hi}}$ is restricted to full episodes $\tau^\text{hi}=(s_0^{\text{hi}},a_0,r^\text{hi}_1,s_1^{\text{hi}},a_1,r^\text{hi}_2,s_2^{\text{hi}},\cdots, s_T^{\text{hi}})$. On the other hand, $\Sigma^{\text{lo}}$ is generative, and simulation can be started by the agent at any state-action pair \cite{kakade2002approximately,kearns2002sparse}. 
Using $\mathcal{T}$ and $\Sigma^\text{lo}$, the agent can map a $\tau^\text{hi}$ to $\tau^\text{lo}=(\mathcal{T}(s_0^{\text{hi}}),a_0,r^\text{lo}_1,\mathcal{T}(s_1^{\text{hi}}),a_1,r^\text{lo}_2,\mathcal{T}(s_2^{\text{hi}}),\cdots, \mathcal{T}(s_T^{\text{hi}}))$, and it is assumed that $\text{Pr}(\tau^{\text{lo}})>0$ under $\mathcal{P}^{\text{lo}}$ and $\vect{\beta}^{\text{lo}}$. It is also assumed that  $\mathcal{R}^{\text{lo}}(\mathcal{T}(s^{\text{hi}}),a)$ and $\mathcal{R}^{\text{hi}}(s^{\text{hi}},a)$ are correlated.



Based on this setup,  a correlation exits between the low- and high- fidelity trajectories that can be beneficial for policy learning. In this work we study how to leverage the cheaply accessible low-fidelity trajectories from $\Sigma^{\text{lo}}$, to learn an optimal $\pi^*$ that maximizes $\mathbb{E}_{s\sim\beta^{\text{hi}}}\bigg[ \mathbb{E}_{a_t\sim \pi, s_t \sim \mathcal{P}^{\text{hi}}}\bigg[\sum_{t=0}^\infty \gamma^t \mathcal{R}^\text{hi}(s_t^{\text{hi}},a_t)|s_0^{\text{hi}} = s \bigg] \bigg]$; in  other  words, to learn $\pi^*$ that is optimal with respect to the high-fidelity environment $\Sigma^{\text{hi}}$.



\begin{wrapfigure}[13]{r}{0.5\textwidth}
\vspace{-6mm}
\includegraphics[width=0.5\textwidth]{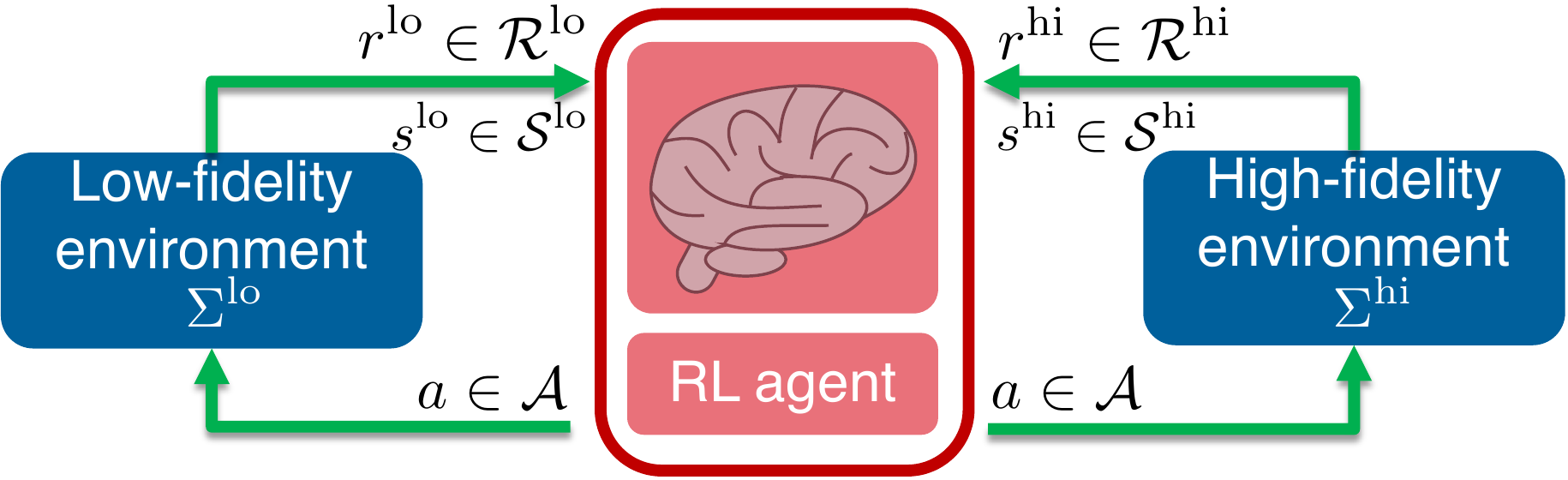}
\caption{RL with low- and high-fidelity environments. 
    $\Sigma^{\text{lo}}$ is cheap to evaluate but is potentially inaccurate. $\Sigma^{\text{hi}}$ represents the real world with the highest accuracy, yet it is expensive to evaluate. The RL agent leverages the correlations between the low- and high-fidelity data to learn $\pi^*_{\text{hi}}$.}
\label{fig:mfrl}
\end{wrapfigure}

\subsection{Multifidelity Monte Carlo RL}
The Monte Carlo method to solve the RL problem is based on the idea of averaging sample returns. In the MC method, experience is divided into episodes. 
At the end of an episode, state-action values are estimated, and the policy is updated. For ease of exposition, we consider a specific state-action pair $(s^{\text{hi}},a)$ in what follows and suppress the dependence on $(s^{\text{hi}},a)$ from the notation to avoid clutter. Consider a sample trajectory $\tau^{\text{hi}}$
that results from the agent's interaction with the high-fidelity environment starting at $(s_0^{\text{hi}}=s^{\text{hi}},a_0=a)$ and following $\pi$, that is,  $\tau^{\text{hi}}: s^{\text{hi}}_{0},a_{0},r^{\text{hi}}_{1},s^{\text{hi}}_{1},a_{1},r^{\text{hi}}_{2},\cdots, s_T^{\text{hi}}$. Note that $r^{\text{hi}}_{t+1}=\mathcal{R}^\text{hi}(s_t^{\text{hi}},a_t)$. Let $\mathcal{G}^{\text{hi}}$
denote the corresponding long-term discounted return, $\mathcal{G}^{\text{hi}} = \sum_{t=0}^\infty \gamma^t r^{\text{hi}}_{t+1}$. The high-fidelity state-action value of the pair $(s,a)$ when the agent follows $\pi$ is 
\begin{equation}
    Q^{\text{hi}}_\pi(s^{\text{hi}},a) = \mathbb{E}_{\tau^{\text{hi}}}\big[\mathcal{G}^{\text{hi}} | s_0^{\text{hi}} = s^{\text{hi}},a_0=a \big]. 
\end{equation}
Notice that $Q^{\text{hi}}_{\pi}(s^{\text{hi}},a)$ is the expectation of an r.v. $\mathcal{G}^\text{hi}$ with respect to the random trajectory $\tau^\text{hi}$. $\mathcal{G}^\text{hi}$ is a bounded r.v. with support  on the interval $[\frac{R_{\text{min}}^{\text{hi}}}{1-\gamma},\frac{R_{\text{max}}^{\text{hi}}}{1-\gamma}]$ and has a finite variance given by 
\begin{equation}
\begin{aligned}
    &\sigma_{\text{hi}}^2{(s^{\text{hi}},a)} = \mathbb{E}_{\tau^{\text{hi}}}\Big[\big(\mathcal{G}^{\text{hi}}-Q^{\text{hi}}_\pi(s^{\text{hi}},a)\big)^2 | s_0 = s^{\text{hi}},a_0=a \Big]. 
\end{aligned}
\end{equation}

By interacting with the environment, the agent can  sample only a finite number of trajectories, $n$. Let $\tau_{1}^{\text{hi}},\tau_{2}^{\text{hi}}, \cdots, \tau_{n}^{\text{hi}}$ be the $n$ sampled trajectories that starts at the pair $(s^{\text{hi}},a)$. Furthermore, let $\mathcal{G}^{\text{hi}}_{1}, \mathcal{G}^{\text{hi}}_{2}, \cdots, \mathcal{G}^{\text{hi}}_{n}$ be i.i.d. r.vs. that correspond to the long-term discounted returns of the sampled trajectories, $\tau_{1}^{\text{hi}},\tau_{2}^{\text{hi}}, \cdots, \tau_{n}^{\text{hi}}$, respectively. Notice that $\mathbb{E}_{\tau^\text{hi}}[\mathcal{G}^{\text{hi}}_{1}]= \mathbb{E}_{\tau^\text{hi}}[\mathcal{G}^{\text{hi}}_{2}]= \cdots= \mathbb{E}_{\tau^\text{hi}}[\mathcal{G}^{\text{hi}}_{n}] = Q^{\text{hi}}_\pi(s,a)$. The first-visit MC sample average is 
\begin{equation} \label{eq:mc}
    \hat{Q}^{\text{hi}}_{\pi,n}(s^{\text{hi}},a) = \frac{1}{n}\sum_{i=1}^{n} \mathcal{G}^{\text{hi}}_{i}.
\end{equation}
By the weak law of large numbers, $
    \underset{{n\rightarrow \infty}}{\text{lim}} \text{Pr}\big(|\hat{Q}^{\text{hi}}_{\pi,n}(s^{\text{hi}},a) -  Q^{\text{hi}}_\pi(s^{\text{hi}},a) | > \xi \big) = 0$,
for any positive number $\xi$. In addition, the variance of this unbiased sample average estimator is
\begin{equation} \label{eq:qvar}
    \text{Var}\Big[\hat{Q}^{\text{hi}}_{\pi,n}(s^{\text{hi}},a)\Big] =  \frac{\sigma_{\text{hi}}^2{(s^{\text{hi}},a)}}{n}.
\end{equation}

Using the low-fidelity generative environment and the method of control variates, we design an unbiased estimator for the expected long-term discounted returns that has a smaller variance than \eqref{eq:qvar}. Let $\tau_{i}^{\text{lo}}$ be the $i$th low-fidelity trajectory that is obtained from $\tau_{i}^{\text{hi}}$ by using $\mathcal{T}$ and the generative low-fidelity environment to evaluate $r^{\text{low}}_{t+1}=\mathcal{R}^\text{lo}(\mathcal{T}(s_t^{\text{hi}}),a_t)$. Let $\mathcal{G}^{\text{lo}}_{i}$ be the r.v. which corresponds to the long-term discounted return of $\tau_{i}^{\text{lo}}$. Notice that $\mathcal{G}^{\text{hi}}_{i}$ and $\mathcal{G}^{\text{lo}}_{i}$ are correlated r.vs. in this multifidelity setup. Based on those low-fidelity trajectories, the low-fidelity first-visit MC sample average is $
    \hat{Q}^{\text{lo}}_{\pi,n}(\mathcal{T}(s^{\text{hi}}),a) = \frac{1}{n}\sum_{i=1}^{n} \mathcal{G}^{\text{lo}}_{i}$ and
has a variance of $\text{Var}\Big[\hat{Q}^{\text{lo}}_{\pi,n}(\mathcal{T}(s^{\text{hi}}),a)\Big] =  \frac{\sigma_{\text{lo}}^2{(\mathcal{T}(s^{\text{hi}}),a)}}{n}$, where $\sigma_{\text{lo}}^2{(\mathcal{T}(s^{\text{hi}}),a)} = \mathbb{E}_{\tau^\text{lo}}\Big[\big(\mathcal{G}^{\text{lo}}-Q^{\text{lo}}_\pi(\mathcal{T}(s^{\text{hi}}),a)\big)^2 | s_0=\mathcal{T}(s^{\text{hi}}),a_0=a \Big]$ and $Q^{\text{lo}}_\pi(\mathcal{T}(s^{\text{hi}}),a)$ is the true population mean. 

Using the method of control variates presented in Subsection \ref{Sec:CV}, we propose the following multifidelity MC estimator: 
\begin{equation} \label{eq:cvmfrl}
\begin{aligned}
\hat{Q}^{\text{MFMC}}_{\pi,n}(s^{\text{hi}},a) = &\hat{Q}^{\text{hi}}_{\pi,n}(s^{\text{hi}},a) + \alpha^*_{s,a}\bigg( Q_\pi^{\text{lo}}(\mathcal{T}(s^{\text{hi}}),a) - \hat{Q}^{\text{lo}}_{\pi,n}(\mathcal{T}(s^{\text{hi}}),a) \bigg),
\end{aligned}
\end{equation}
where
\begin{equation}
    \alpha^*_{s,a} = \frac{\text{Cov}\big[\hat{Q}^{\text{hi}}_{\pi,n}(s^{\text{hi}},a),\hat{Q}^{\text{lo}}_{\pi,n}(\mathcal{T}(s^{\text{hi}}),a) \big]}{\text{Var}\big[ \hat{Q}^{\text{lo}}_{\pi,n}(\mathcal{T}(s^{\text{hi}}),a)\big]}.
\end{equation}
Notice that the estimator in \eqref{eq:cvmfrl} is unbiased and has a variance of
\begin{equation}
     \text{Var}\Big[\hat{Q}^{\text{MFMC}}_{\pi,n}(s^{\text{hi}},a)\Big] =  \big(1-\rho_{s,a}^2\big)\text{Var}\Big[\hat{Q}^{\text{hi}}_{\pi,n}(s^{\text{hi}},a)\Big], 
\end{equation}
where $\rho_{s,a}$ is the correlation coefficient between the low-fidelity and high-fidelity long-term discounted returns: 
\begin{equation}
   \rho_{s,a} = \frac{\text{Cov}\big[\hat{Q}^{\text{hi}}_{\pi,n}(s^{\text{hi}},a),\hat{Q}^{\text{lo}}_{\pi,n}(\mathcal{T}(s^{\text{hi}}),a) \big]}{\sqrt{\text{Var}\Big[\hat{Q}^{\text{hi}}_{\pi,n}(s^{\text{hi}},a)\Big]\text{Var}\Big[\hat{Q}^{\text{lo}}_{\pi,n}(\mathcal{T}(s^{\text{hi}}),a)\Big]}}.
\end{equation}
Therefore, the variance in estimating the value of a state-action pair under a policy $\pi$ can be reduced by a factor of $\big(1-\rho_{s,a}^2\big)$ when the low-fidelity data is exploited, although the budget of high-fidelity samples remains the same. Notice that
\begin{equation}
\begin{aligned}
   \text{Cov}\big[&\hat{Q}^{\text{hi}}_{\pi,n}(s^{\text{hi}},a),\hat{Q}^{\text{lo}}_{\pi,n}(\mathcal{T}(s^{\text{hi}}),a) \big] =  \text{Cov}\big[\frac{1}{n}\sum_{i=1}^{n} \mathcal{G}^{\text{hi}}_{i},\frac{1}{n}\sum_{i=1}^{n} \mathcal{G}^{\text{lo}}_{i} \big]
    = \frac{1}{n}
   \text{Cov}\big[ \mathcal{G}^{\text{hi}}_{i},\mathcal{G}^{\text{lo}}_{i} \big],
\end{aligned}
\end{equation}
because $\mathcal{G}^{\text{hi}}_{i},\mathcal{G}^{\text{lo}}_{j}$ are independent r.vs. $\forall i \neq j$. Hence,  $\text{Cov}\big[\hat{Q}^{\text{hi}}_{\pi,n}(s^{\text{hi}},a),\hat{Q}^{\text{lo}}_{\pi,n}(\mathcal{T}(s^{\text{hi}}),a) \big]$, $\text{Var}\Big[\hat{Q}^{\text{hi}}_{\pi,n}(s^{\text{hi}},a)\Big]$, and $\text{Var}\Big[\hat{Q}^{\text{lo}}_{\pi,n}(\mathcal{T}(s^{\text{hi}}),a)\Big]$ can all be estimated in practice based on the return data samples using the standard unbiased estimators for the variance and covariance. 

The reduced-variance estimator of \eqref{eq:cvmfrl} can be used to design a multifidelity Monte Carlo RL algorithm as shown in Algorithm 1 in Appendix A. This algorithm is based on the on-policy first-visit MC control algorithm with $\epsilon$-soft policies  \cite{sutton2018reinforcement} but uses the multifidelity estimator \eqref{eq:cvmfrl}. Algorithm 1 is based on the idea of generalized policy iteration. In the policy evaluation step (lines 11--18), the state-action value function is made consistent with the current policy by updating the estimated long-term discounted returns of a state-action pair $(s_t,a_t)$ using the control-variate-based estimator \eqref{eq:cvmfrl} (line 18). This update requires the estimation of the correlation between the low- and high- fidelity returns, which is done in lines 13--17. Next, in the policy improvement step (lines 19--20), the policy is made $\epsilon$-greedy with respect to the current state-action value function. In each episode, the agent needs to evaluate the policy in the low-fidelity environment to obtain $Q^\text{lo}_\pi$. This can be done in practice by collecting a large number of $m$ return samples from the cheap low-fidelity environment and setting $Q^\text{lo}_\pi(\mathcal{T}(s^{\text{hi}}),a) \approx \hat{Q}_{\pi,m+n}^{\text{lo}}(\mathcal{T}(s^{\text{hi}}),a)$.
The convergence of Algorithm 1 to the optimal $\epsilon$-greedy policy, $\pi^*_{\epsilon-\text{opt}}$,  along with its corresponding $\hat{Q}_*^\text{MFMC}$, is guaranteed under the same conditions that guarantee convergence for the on-policy first-visit MC control algorithm with $\epsilon$-soft policies  \cite{sutton2018reinforcement}. In the following subsection, we theoretically analyze the impacts of variance reduction on policy evaluation and policy improvement. 

\subsection{Theoretical analysis}
In this subsection we analyze the impacts of variance reduction on policy evaluation error and policy improvement by introducing two main theorems. Intermediate lemmas along with all the proofs can be found in Appendix B.

\subsubsection{Policy evaluation}
In policy evaluation, the task is to estimate the state-action value function of a given policy $\pi$. Trajectory samples are first generated by interacting with the environment using $\pi$, and the state-action value function is then estimated using either the single high-fidelity estimator \eqref{eq:mc} or the proposed multifidelity estimator \eqref{eq:cvmfrl}. To analyze the impacts of variance  reduction on policy evaluation error, we first derive a a Bernstein-type concentration inequality \cite{Bernstein} that relates the deviation between the sample average and the true mean to the sample size $n$, estimation accuracy parameters $\delta, \xi$, and the variance of a r.v. as follows.
\begin{lemma}
Let $X_1, X_2, \cdots, X_n$ be i.i.d. r.vs. with mean $\mathbb{E}[X_i]=\mu_{_X}$ and variance $\mathbb{E}[(X_i-\mu_{_X})^2]=\sigma_{_X}^2$, $\forall i\in [n]$. Furthermore, suppose that $X_i, \forall i$, are bounded almost surely with a parameter $b$, namely,  $\text{Pr}(|X_i - \mu_{_X}| \leq b)=1, \forall i$. Then 
\begin{equation} \label{eq:chernof}
    \text{Pr}\Bigg(\left\lvert\frac{1}{n}\sum_{i=1}^n X_i - \mu_{_X} \right\lvert \geq \xi \Bigg) \leq 2\text{exp}\bigg(\frac{-n\xi^2}{4\sigma_{_X}^2} \bigg)
\end{equation}
for $0\leq \xi \leq \sigma_{_X}^2/b$.
\end{lemma}
Next, the concentration bound of Lemma 1 is used to derive the minimum sample size that is required to ensure that the sample average deviates by no more than $\xi$ from the true mean with high probability for both the high-fidelity estimator \eqref{eq:mc} and the multifidelity estimator \eqref{eq:cvmfrl}.

\begin{theorem}
To guarantee that
\begin{enumerate}[leftmargin=0.6cm]
    \item $\text{Pr}\Big(|\hat{Q}^{\text{hi}}_{\pi,n}(s^{\text{hi}},a) - {Q}^{\text{hi}}_{\pi}(s^{\text{hi}},a) | \leq \xi \Big) \geq 1-\delta $,  then $n\geq \frac{4 \sigma_{\text{hi}}^2{(s^{\text{hi}},a)}}{\xi^2} \text{log}(\frac{2}{\delta})$.
    \item $\text{Pr}\Big(|\hat{Q}^{\text{MFMC}}_{\pi,n}(s,a) - {Q}^{\text{hi}}_{\pi}(s^{\text{hi}},a) | \leq \xi \Big) \geq 1-\delta$, then $n\geq \frac{4(1-\rho_{s,a}^2)\sigma_{\text{hi}}^2{(s^{\text{hi}},a)}}{\xi^2} \text{log}(\frac{2}{\delta})$.
\end{enumerate}
\end{theorem}
The result of Theorem 1 highlights the benefit of using our proposed multifidelity estimator \eqref{eq:cvmfrl} for policy evaluation as opposed to using the single high-fidelity estimator of \eqref{eq:mc}. By leveraging the correlation between low- and high-fidelity returns $\rho_{s,a}$, the variance of the multifidelity estimator is reduced by a factor of $(1-\rho_{s,a}^2)$, which makes it possible to achieve a low estimation error at a reduced number of high-fidelity samples. 

\subsubsection{Policy improvement}
In policy improvement, a new policy $\pi^\prime$ is constructed by deterministically choosing the greedy action  with respect to the state-action value function of the original policy $\pi$, $Q_\pi^{\text{hi}}(s,a)$, at every state, that is, $\pi^\prime(s) \overset{.}{=} \underset{{a \in \mathcal{A}}}{\text{argmax}}~Q_\pi^{\text{hi}}(s,a), \forall s \in \mathcal{S} $. By the policy improvement theorem, $\pi^\prime$ is as good as or better than $\pi$ under the assumption that $Q_\pi^{\text{hi}}(s,a), \forall s,a$ is computed exactly.
In practice, the MDP is unknown, and the state-action value function is estimated based on a finite number of trajectories. Moreover, those trajectories are generated by following an exploratory policy, such as an $\epsilon$-soft policy. Because we are interested in studying how different estimators impact policy improvement, we consider a target state $s^{\text{hi}} \in \mathcal{S}^{\text{hi}}$ and assume that we have $n$ trajectories for each action $a\in \mathcal{A}$ at this target state. This assumption basically ensures that all actions at the target state $s^{\text{hi}}$ have been explored equally well and enables us to make fair comparisons about estimator performance.

Without loss of generality, suppose that $Q_\pi^{\text{hi}}(s^{\text{hi}},a_1) \geq Q_\pi^{\text{hi}}(s^{\text{hi}},a_2) \geq \cdots Q_\pi^{\text{hi}}(s^{\text{hi}},a_{|\mathcal{A}|})$. 
Let $\Delta_i = Q_\pi^{\text{hi}}(s^{\text{hi}},a_1)- Q_\pi^{\text{hi}}(s^{\text{hi}},a_i),\forall i \neq 1$. 
We analyze the probability that $a_1$, which is the greedy action given the true $Q_\pi^{\text{hi}}(s^{\text{hi}},a)$, is the greedy action with respect to the single- and multifidelity estimators in our next theorem. 
\begin{theorem}
Suppose that the number of trajectories from a state-action pair at a target state $s^{\text{hi}} \in \mathcal{S}^{\text{hi}}$ is the same for all actions $a \in \mathcal{A}$ and that $a_1$ is the greedy action with respect to the true $Q^\text{hi}_\pi(s^{\text{hi}},a)$. Furthermore, suppose that 
$\mathcal{P}^{\text{hi}}(s^{\text{hi}}|s^{\text{hi}^\prime},a)\geq \beta(s^{\text{hi}}), \forall s^{\text{hi}} \in \mathcal{S}^{\text{hi}}$. Then
\begin{enumerate}[leftmargin=0.6cm]
    \item $\text{Pr}\big(a_1= \underset{{a \in \mathcal{A}}}{\text{argmax}}~\hat{Q}^{\text{hi}}_{\pi,n}(s^{\text{hi}},a) \big) \geq \prod_{i=2}^{|\mathcal{A}|} \frac{\Delta^2_i}{\Delta^2_i+\text{Var}[\hat{Q}^{\text{hi}}_{\pi,n}(s^{\text{hi}},a_1)]+\text{Var}[\hat{Q}^{\text{hi}}_{\pi,n}(s^{\text{hi}},a_i)]}$.
    \item $\text{Pr}\big(a_1= \underset{{a \in \mathcal{A}}}{\text{argmax}}~\hat{Q}^{\text{MFMC}}_{\pi,n}(s^{\text{hi}},a) \big) \geq \prod_{i=2}^{|\mathcal{A}|} \frac{\Delta^2_i}{\Delta^2_i+(1-\rho^2_{s,a_1})\text{Var}[\hat{Q}^{\text{hi}}_{\pi,n}(s^{\text{hi}},a_i)]+(1-\rho^2_{s,a_i})\text{Var}[\hat{Q}^{\text{hi}}_{\pi,n}(s^{\text{hi}},a_i)]}$.
\end{enumerate}
\end{theorem}

Notice that when $|\rho_{s,a_2}| \rightarrow 1$, the lower bound in the result of Theorem $2$ approaches $1$, which means that the correct greedy action $a_1$ can be selected with certainty when the reduced-variance multifidelity estimator \eqref{eq:cvmfrl} is adopted. Combining the results of Theorems 1 and 2, the proposed \texttt{MFMCRL} algorithm is expected to outperform its single high-fidelity Monte Carlo counterpart in terms of learning a better policy under a given budget of high-fidelity environment interactions. 

\section{Numerical experiments}
In this section we  empirically evaluate the performance of the proposed \texttt{MFMCRL} algorithm on synthetic MDP problems and on a NAS use case.
Our codes and all experimental details can be found in Appendix C.
\subsection{Synthetic MDPs}
We synthesize multifidelity random MDP problems with state space cardinality $|\mathcal{S}|$ and action space cardinality $|\mathcal{A}|$. The high-fidelity transition and reward functions, $\mathcal{P}^{\text{hi}}$ and $\mathcal{R}^{\text{hi}}$, respectively, are first generated based on a random process as detailed in Appendix C.2. Next, for a given $\mathcal{P}^{\text{hi}}$ and $\mathcal{R}^{\text{hi}}$, the corresponding $\mathcal{P}^{\text{low}}$ and $\mathcal{R}^{\text{low}}$ are generated by injecting Gaussian noise to meet a desired signal-to-noise  ratio. Specifically,  we generate a random matrix $\mathcal{P}_N$ of size $|\mathcal{S}|\times |\mathcal{A}| \times |\mathcal{S}|$ from a normally distributed r.v. with mean $0$ and variance $\sigma^2_\mathcal{P}$, and set $\mathcal{P}^{\text{low}} = \mathcal{P}^{\text{hi}} + \mathcal{P}_N$. $\mathcal{P}^{\text{low}}$ is then appropriately normalized so that $\sum_{s^{\text{lo}^\prime} \in \mathcal{S}} \mathcal{P}^\text{lo}(s^{\text{lo}^\prime}|s^\text{lo},a) = 1$. Similarly, we generate a random matrix $\mathcal{R}_N$ of size $|\mathcal{S}| \times |\mathcal{A}|$ from a normally distributed r.v. with mean $0$ and variance $\sigma^2_\mathcal{R}$ and set $\mathcal{R}^{\text{low}} = \mathcal{R}^{\text{hi}} + \mathcal{R}_N$. $\mathcal{P}^{\text{hi}}$ and $\mathcal{R}^{\text{hi}}$ are then encapsulated within a gym-like environment with which the agent can interact by exchanging sample tuples of the form $(s^\text{hi}, a, r^\text{hi},s^{\text{hi}^\prime})$. Similarly,  $\mathcal{P}^{\text{lo}}$ and $\mathcal{R}^{\text{lo}}$ are encapsulated within a gym-like environment to form the low-fidelity environment. In this experiment, both low- and high-fidelity environments share the same state-action space---that is, $\mathcal{T}$ is an identity transformation---yet the transition and reward functions of the low-fidelity environment are different since they are corrupted with noise. Notice that even if the agent could draw an infinite number of samples from $\mathcal{P}^{\text{lo}}$ and $\mathcal{R}^{\text{lo}}$, it would not be able to recover $\mathcal{P}^{\text{hi}}$ and $\mathcal{R}^{\text{hi}}$ since $\mathcal{P}^{\text{lo}}$ and $\mathcal{R}^{\text{lo}}$ underneath the low-fidelity environment themselves are corrupted. This situation mimics what happens in practice when we attempt to learn $\mathcal{P}^{\text{lo}}$ and $\mathcal{R}^{\text{lo}}$ based on real data and build an RL environment off those learned functions to train the agent. 

After constructing the multifidelity environments, we train an RL agent using the proposed \texttt{MFMCRL} algorithm over $10\text{K}$ high-fidelity episodes, where a training episode is defined to be a trajectory that ends at a terminal state. The \texttt{MFMCRL} agent interacts with the low-fidelity environment as shown in Algorithm 1, to generate reduced-variance estimates of the state-action value function.  As a baseline for comparison, we train another RL agent (\texttt{MCRL}) using the standard the first-visit MC control algorithm over $10\text{K}$ high-fidelity episodes \cite{sutton2018reinforcement}. We set $\gamma$ and $\epsilon$  to $0.99$ and $0.1$, respectively. Every $50$ training episodes, the greedy policy w.r.t to the estimated $Q$ function is used to test the performance of the agent on $200$ test episodes. We repeat the whole experiment with $36$ different random seeds (to fully leverage our $36$ core machine) and report the mean and standard deviation (across different seeds) of the test episode rewards in Figure \ref{fig:Perf}(a).
One can  observe that for a given budget of high-fidelity episodes, the proposed \texttt{MFMCRL} algorithm outperforms \texttt{MCRL} in terms of policy performance, with performance improving as the RL agent collects more low-fidelity samples ($\#\tau^\text{lo}$ refers to the number of low-fidelity trajectories started from a state-action pair). In Figure \ref{fig:Perf}(b), we vary the SNR of the low-fidelity environment and observe that performance improves as SNR increases. This is expected because the low- and high-fidelity environments are better correlated at higher SNRs. 
Notice that when the SNR of the low-fidelity environment is -10 dB, there is no benefit from doing multifidelity RL. The reason is that the low- and high-fidelity environments are too weakly correlated to benefit from multifidelity estimation. In fact, for this case $\mathbb{E}_{s,a,s^\prime}[|\mathcal{P}^\text{hi}-\mathcal{P}^\text{lo}|]=
0.275\pm0.33$, and $\mathbb{E}_{s,a}[|\mathcal{R}^\text{hi}-\mathcal{R}^\text{lo}|]=1.029 \pm 0.024$, compared with the other extreme case (SNR +3dB) for which $\mathbb{E}_{s,a,s^\prime}[|\mathcal{P}^\text{hi}-\mathcal{P}^\text{lo}|]=
0.009\pm0.0002$, and $\mathbb{E}_{s,a}[|\mathcal{R}^\text{hi}-\mathcal{R}^\text{lo}|]=0.230 \pm 0.006$. This is also evident in Figure \ref{fig:Perf}(c), where we show the mean variance reduction factor $\mathrm{Var}[\hat{Q}^{\mathrm{MFMC}}]/\mathrm{Var}[\hat{Q}^{\mathrm{hi}}]$ estimated based off the last $1$K training episodes. When the low-fidelity environment is less noisy (higher SNR), more variance reduction can be attained.

\begin{figure}[h] 
\centering
\subfloat[$\Sigma^\text{lo}$ has an SNR of -3dB]{\includegraphics[width=0.33\columnwidth]{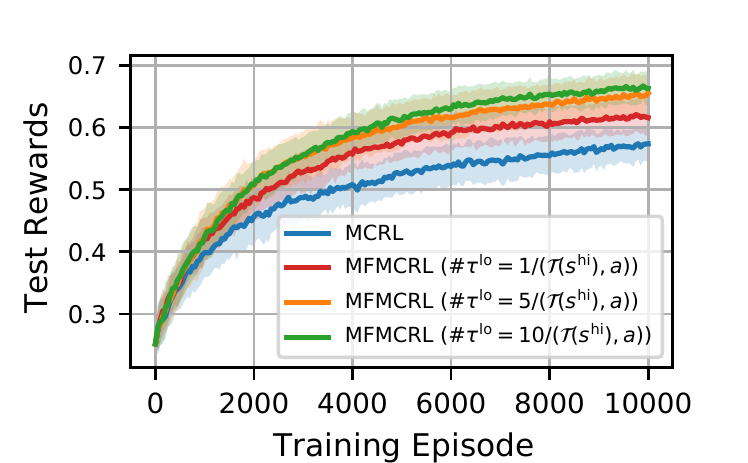}}
\hfill
\centering
\subfloat[$\#\tau^\mathrm{lo} = 10/(\mathcal{T}(s^\mathrm{hi}),a))$]{\includegraphics[width=0.33\columnwidth]{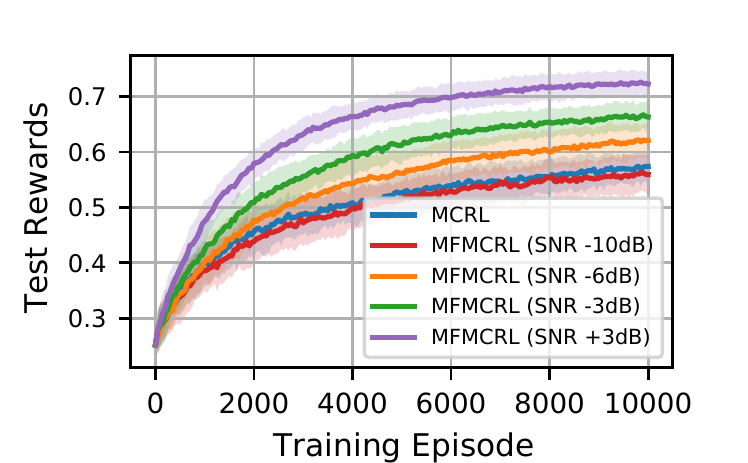}} 
\hfill
\centering
\subfloat[Variance reduction factor]{\includegraphics[width=0.33\columnwidth]{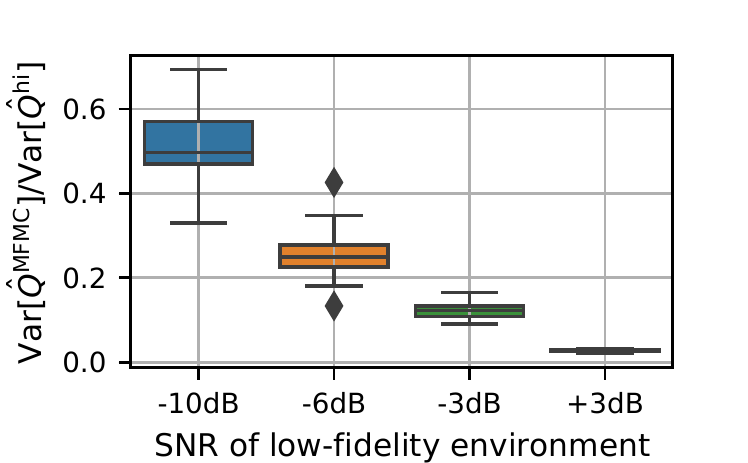}} 
\hfill
\caption{Mean and standard deviation of test episode rewards for the proposed \texttt{MFMCRL} during training: (a) test episode rewards improve with increasing number of low-fidelity samples ($\#\tau^\text{lo}$); (b) test episode rewards improve with less noisy low-fidelity environments; (c) variance reduction factor improves when low- and high-fidelity environments are more correlated. These results are based on a random MDP with $|\mathcal{S}|=200, |\mathcal{A}|=8$. 
}
\label{fig:Perf}
\end{figure}
\vspace{-3mm}
\subsection{NAS}
In NAS, the task is to discover high-performing neural architectures with respect to a given training dataset over a predefined search space. While many earlier works attempted to design RL-based NAS algorithms, \cite{baker2016designing, zoph2016neural, jaafra2018review}, it has since become clear that the sample complexity of RL is too high to be competitive with state-of-the-art NAS methods \cite{balaprakash2019scalable, white2019bananas}. In this experiment we  study how multifidelity RL can improve learning in NAS over standard RL, which could serve to catalyze future work in this direction to make RL more competitive in NAS.

For this experiment we use the tabular dataset of NAS-Bench-201 \cite{dong2020bench} to construct multifidelity RL environments as detailed in Appendix C.3. In summary, the RL agent sequentially configures the nodes of an architecture (inducing an MDP), after which the architecture is trained on the training dataset for $L$ epochs, and the validation accuracy on a held-out validation data set is provided to the agent as a reward. By maximizing the total rewards, high-performing architectures can be discovered. NAS-Bench-201 reports the validation accuracy curves for all the architectures in the search space
\begin{wrapfigure}[20]{r}{0.5\textwidth}
\includegraphics[width=0.5\textwidth]{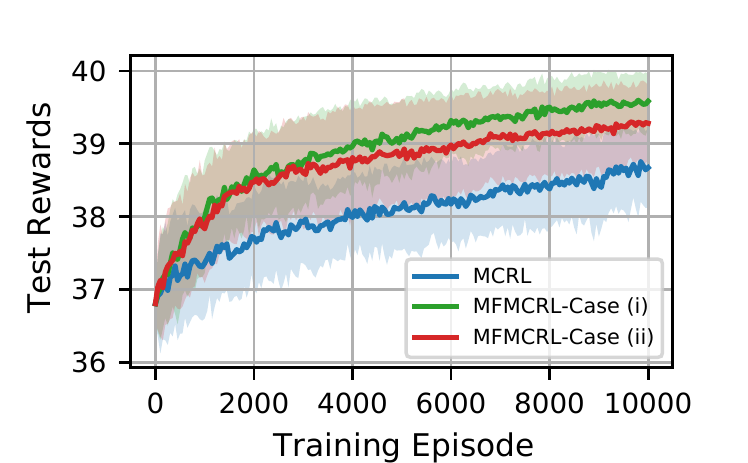}
\caption{Mean and standard deviation of test episode rewards for the proposed \texttt{MFMCRL} during training on multifidelity NAS environments. See text for description of the two multifidelity scenarios (i) and (ii). In both cases, $\#\tau^\mathrm{lo} = 5/(\mathcal{T}(s^\mathrm{hi}),a))$. }
\label{fig:nas}
\end{wrapfigure}
as a function of the number of training epochs and for three image data sets. We construct two multifidelity scenarios as follows. In both scenarios, the validation accuracy of an architecture at the end of training (i.e., at $L=200$ epochs) is used as a high-fidelity reward in the high-fidelity environment. For the low-fidelity environment, we have two cases: (i) low-fidelity environment is identical to the high-fidelity environment except for the reward function, which is now the validation accuracy at the $L=10$th training epoch, and  (ii) low-fidelity environment is defined for a smaller search space and the reward function is the validation accuracy of an architecture at the $L=10$th training epoch. Note that in case (ii) the state space and dynamics differ between the low- and high-fidelity environments. For both cases, we train both our proposed \texttt{MFMCRL} and the \texttt{MCRL} exactly as we did in Section 4.1, and we report the mean and standard deviation of test episode rewards in Figure  \ref{fig:nas}. We can observe that our multifidelity RL framework does indeed improve over standard RL and that performance gains are higher when the low- and high-fidelity environments are more similar, case (i).

\vspace{-3mm}
\section{Conclusion}
In this paper we have studied the RL problem in the presence of a low- and a high-fidelity environment for a given control task, with the aim of improving the agent's performance in the high-fidelity environment with multifidelity data. We have proposed a multifidelity estimator based on the method of control variates, which uses low-fidelity data to reduce the variance in the estimation of the state-action value function.  The impacts of variance reduction on policy improvement and policy evaluation are theoretically analyzed, and a multifidelity Monte Carlo RL algorithm (\texttt{MFMCRL}) is devised. We show that for a finite budget of high-fidelity data, the \texttt{MFMCRL} agent can well exploit the cross-correlations between low- and high-fidelity data and yield superior performance. In our future work, we will study the design of a control-variate-based multifidelity RL framework with function approximation to solve continuous state-action space RL problems. 

\section{Broader impact}
\textit{Positive impacts:} The energy/cost associated with generating low-fidelity data is generally much smaller than that of high-fidelity data. By leveraging low-fidelity data to improve the learning of RL agents, greener agents are realized. \textit{Negative impacts:} Running multifidelity RL agent training with weakly-correlated low- and high-fidelity environments  can be wasteful of resources since the benefits in this case are not significant. 

\bibliography{references}

\begin{thebibliography}{}

\bibitem[Abbeel et~al., 2006]{abbeel2006using}
Abbeel, P., Quigley, M., and Ng, A.~Y. (2006).
\newblock Using inaccurate models in reinforcement learning.
\newblock In {\em Proceedings of the 23rd international conference on Machine
  Learning}, pages 1--8.

\bibitem[Altman, 1999]{altman1999constrained}
Altman, E. (1999).
\newblock {\em Constrained {Markov} decision processes}, volume~7.
\newblock CRC Press.

\bibitem[Baker et~al., 2016]{baker2016designing}
Baker, B., Gupta, O., Naik, N., and Raskar, R. (2016).
\newblock Designing neural network architectures using reinforcement learning.
\newblock {\em arXiv preprint arXiv:1611.02167}.

\bibitem[Balaprakash et~al., 2019]{balaprakash2019scalable}
Balaprakash, P., Egele, R., Salim, M., Wild, S., Vishwanath, V., Xia, F.,
  Brettin, T., and Stevens, R. (2019).
\newblock Scalable reinforcement-learning-based neural architecture search for
  cancer deep learning research.
\newblock In {\em Proceedings of the International Conference for High
  Performance Computing, Networking, Storage and Analysis}, pages 1--33.

\bibitem[Bellman, 1957]{bellman1957markovian}
Bellman, R. (1957).
\newblock A {Markovian} decision process.
\newblock {\em Journal of Mathematics and Mechanics}, 6(5):679--684.

\bibitem[Bernstein, 1924]{Bernstein}
Bernstein, S. (1924).
\newblock On a modification of {Chebyshev’s inequality and of the error
  formula of Laplace}.
\newblock {\em Ann. Sci. Inst. Sav. Ukraine, Sect. Math. 1, 4(5)}.

\bibitem[Bertsekas et~al., 2000]{bertsekas2000dynamic}
Bertsekas, D.~P. et~al. (2000).
\newblock {\em Dynamic programming and optimal control: {Vol. 1}}.
\newblock Athena Scientific Belmont.

\bibitem[Chebotar et~al., 2019]{chebotar2019closing}
Chebotar, Y., Handa, A., Makoviychuk, V., Macklin, M., Issac, J., Ratliff, N.,
  and Fox, D. (2019).
\newblock Closing the sim-to-real loop: Adapting simulation randomization with
  real world experience.
\newblock In {\em 2019 International Conference on Robotics and Automation
  (ICRA)}, pages 8973--8979. IEEE.

\bibitem[Cutler et~al., 2015]{cutler2015real}
Cutler, M., Walsh, T.~J., and How, J.~P. (2015).
\newblock Real-world reinforcement learning via multifidelity simulators.
\newblock {\em IEEE Transactions on Robotics}, 31(3):655--671.

\bibitem[Dong and Yang, 2020]{dong2020bench}
Dong, X. and Yang, Y. (2020).
\newblock {NAS-Bench-201}: Extending the scope of reproducible neural
  architecture search.
\newblock {\em arXiv preprint arXiv:2001.00326}.

\bibitem[Fern{\'a}ndez-Godino et~al., 2016]{fernandez2016review}
Fern{\'a}ndez-Godino, M.~G., Park, C., Kim, N.-H., and Haftka, R.~T. (2016).
\newblock Review of multi-fidelity models.
\newblock {\em arXiv preprint arXiv:1609.07196}.

\bibitem[Haarnoja et~al., 2018]{pmlr-v80-haarnoja18b}
Haarnoja, T., Zhou, A., Abbeel, P., and Levine, S. (2018).
\newblock Soft actor-critic: Off-policy maximum entropy deep reinforcement
  learning with a stochastic actor.
\newblock In Dy, J. and Krause, A., editors, {\em Proceedings of the 35th
  International Conference on Machine Learning}, volume~80 of {\em Proceedings
  of Machine Learning Research}, pages 1861--1870. PMLR.

\bibitem[Jaafra et~al., 2018]{jaafra2018review}
Jaafra, Y., Laurent, J.~L., Deruyver, A., and Naceur, M.~S. (2018).
\newblock A review of meta-reinforcement learning for deep neural networks
  architecture search.
\newblock {\em arXiv preprint arXiv:1812.07995}.

\bibitem[Janner et~al., 2019]{janner2019trust}
Janner, M., Fu, J., Zhang, M., and Levine, S. (2019).
\newblock When to trust your model: Model-based policy optimization.
\newblock {\em Advances in Neural Information Processing Systems},
  32:12519--12530.

\bibitem[Kakade and Langford, 2002]{kakade2002approximately}
Kakade, S. and Langford, J. (2002).
\newblock Approximately optimal approximate reinforcement learning.
\newblock In {\em In Proc. 19th International Conference on Machine Learning}.
  Citeseer.

\bibitem[Kallenberg, 2011]{kallenberg2011markov}
Kallenberg, L. (2011).
\newblock Markov decision processes.
\newblock {\em Lecture Notes. University of Leiden}.

\bibitem[Kearns et~al., 2002]{kearns2002sparse}
Kearns, M., Mansour, Y., and Ng, A.~Y. (2002).
\newblock A sparse sampling algorithm for near-optimal planning in large
  {Markov} decision processes.
\newblock {\em Machine Learning}, 49(2):193--208.

\bibitem[Khairy et~al., 2020]{khairy2020learning}
Khairy, S., Shaydulin, R., Cincio, L., Alexeev, Y., and Balaprakash, P. (2020).
\newblock Learning to optimize variational quantum circuits to solve
  combinatorial problems.
\newblock In {\em AAAI}, pages 2367--2375.

\bibitem[Koutsourelakis, 2009]{koutsourelakis2009accurate}
Koutsourelakis, P.-S. (2009).
\newblock Accurate uncertainty quantification using inaccurate computational
  models.
\newblock {\em SIAM Journal on Scientific Computing}, 31(5):3274--3300.

\bibitem[Lemieux, 2014]{lemieux2014control}
Lemieux, C. (2014).
\newblock Control variates.
\newblock {\em Wiley StatsRef: Statistics Reference Online}, pages 1--8.

\bibitem[Li et~al., 2020]{li2020multi}
Li, S., Xing, W., Kirby, R., and Zhe, S. (2020).
\newblock Multi-fidelity {Bayesian} optimization via deep neural networks.
\newblock {\em Advances in Neural Information Processing Systems}, 33.

\bibitem[Lillicrap et~al., 2015]{lillicrap2015continuous}
Lillicrap, T.~P., Hunt, J.~J., Pritzel, A., Heess, N., Erez, T., Tassa, Y.,
  Silver, D., and Wierstra, D. (2015).
\newblock Continuous control with deep reinforcement learning.
\newblock {\em arXiv preprint arXiv:1509.02971}.

\bibitem[Mann and Choe, 2013]{mann2013directed}
Mann, T.~A. and Choe, Y. (2013).
\newblock Directed exploration in reinforcement learning with transferred
  knowledge.
\newblock In {\em European Workshop on Reinforcement Learning}, pages 59--76.
  PMLR.

\bibitem[Meng and Karniadakis, 2020]{meng2020composite}
Meng, X. and Karniadakis, G.~E. (2020).
\newblock A composite neural network that learns from multi-fidelity data:
  Application to function approximation and inverse {PDE} problems.
\newblock {\em Journal of Computational Physics}, 401:109020.

\bibitem[Mnih et~al., 2015]{mnih2015human}
Mnih, V., Kavukcuoglu, K., Silver, D., Rusu, A.~A., Veness, J., Bellemare,
  M.~G., Graves, A., Riedmiller, M., Fidjeland, A.~K., Ostrovski, G., et~al.
  (2015).
\newblock Human-level control through deep reinforcement learning.
\newblock {\em Nature}, 518(7540):529--533.

\bibitem[Moster et~al., 2021]{moster2021galaxynet}
Moster, B.~P., Naab, T., Lindstr{\"o}m, M., and O’Leary, J.~A. (2021).
\newblock {GalaxyNet}: connecting galaxies and dark matter haloes with deep
  neural networks and reinforcement learning in large volumes.
\newblock {\em Monthly Notices of the Royal Astronomical Society},
  507(2):2115--2136.

\bibitem[Peherstorfer et~al., 2016]{peherstorfer2016optimal}
Peherstorfer, B., Willcox, K., and Gunzburger, M. (2016).
\newblock Optimal model management for multifidelity {Monte Carlo} estimation.
\newblock {\em SIAM Journal on Scientific Computing}, 38(5):A3163--A3194.

\bibitem[Peherstorfer et~al., 2018]{peherstorfer2018survey}
Peherstorfer, B., Willcox, K., and Gunzburger, M. (2018).
\newblock Survey of multifidelity methods in uncertainty propagation,
  inference, and optimization.
\newblock {\em Siam Review}, 60(3):550--591.

\bibitem[Perdikaris et~al., 2017]{perdikaris2017nonlinear}
Perdikaris, P., Raissi, M., Damianou, A., Lawrence, N.~D., and Karniadakis,
  G.~E. (2017).
\newblock Nonlinear information fusion algorithms for data-efficient
  multi-fidelity modelling.
\newblock {\em Proceedings of the Royal Society A: Mathematical, Physical and
  Engineering Sciences}, 473(2198):20160751.

\bibitem[Schrittwieser et~al., 2020]{schrittwieser2020mastering}
Schrittwieser, J., Antonoglou, I., Hubert, T., Simonyan, K., Sifre, L.,
  Schmitt, S., Guez, A., Lockhart, E., Hassabis, D., Graepel, T., et~al.
  (2020).
\newblock Mastering {A}tari, {G}o, chess and shogi by planning with a learned
  model.
\newblock {\em Nature}, 588(7839):604--609.

\bibitem[Schulman et~al., 2015]{schulman2015trust}
Schulman, J., Levine, S., Abbeel, P., Jordan, M., and Moritz, P. (2015).
\newblock Trust region policy optimization.
\newblock In {\em International conference on Machine Learning}, pages
  1889--1897. PMLR.

\bibitem[Schulman et~al., 2017]{schulman2017proximal}
Schulman, J., Wolski, F., Dhariwal, P., Radford, A., and Klimov, O. (2017).
\newblock Proximal policy optimization algorithms.
\newblock {\em arXiv preprint arXiv:1707.06347}.

\bibitem[Suryan et~al., 2020]{suryan2020multifidelity}
Suryan, V., Gondhalekar, N., and Tokekar, P. (2020).
\newblock Multifidelity reinforcement learning with {Gaussian} processes:
  model-based and model-free algorithms.
\newblock {\em IEEE Robotics \& Automation Magazine}, 27(2):117--128.

\bibitem[Sutton and Barto, 2018]{sutton2018reinforcement}
Sutton, R.~S. and Barto, A.~G. (2018).
\newblock {\em Reinforcement learning: An introduction}.
\newblock MIT Press.

\bibitem[Taylor and Stone, 2009]{taylor2009transfer}
Taylor, M.~E. and Stone, P. (2009).
\newblock Transfer learning for reinforcement learning domains: A survey.
\newblock {\em Journal of Machine Learning Research}, 10(7).

\bibitem[Taylor et~al., 2007]{taylor2007transfer}
Taylor, M.~E., Stone, P., and Liu, Y. (2007).
\newblock Transfer learning via inter-task mappings for temporal difference
  learning.
\newblock {\em Journal of Machine Learning Research}, 8(9).

\bibitem[White et~al., 2019]{white2019bananas}
White, C., Neiswanger, W., and Savani, Y. (2019).
\newblock Bananas: {Bayesian} optimization with neural architectures for neural
  architecture search.
\newblock {\em arXiv preprint arXiv:1910.11858}, 1(2):4.

\bibitem[Williams and Rasmussen, 2006]{williams2006gaussian}
Williams, C.~K. and Rasmussen, C.~E. (2006).
\newblock {\em Gaussian processes for machine learning}, volume~2.
\newblock MIT Press, Cambridge, MA.

\bibitem[Zhu et~al., 2020]{zhu2020transfer}
Zhu, Z., Lin, K., and Zhou, J. (2020).
\newblock Transfer learning in deep reinforcement learning: A survey.
\newblock {\em arXiv preprint arXiv:2009.07888}.

\bibitem[Zoph and Le, 2016]{zoph2016neural}
Zoph, B. and Le, Q.~V. (2016).
\newblock Neural architecture search with reinforcement learning.
\newblock {\em arXiv preprint arXiv:1611.01578}.

\end{thebibliography}

\clearpage

\appendix

\section{Proposed \texttt{MFMCRL} algorithm}

\begin{algorithm}[H]
    \SetKwInOut{Input}{Input}
    \SetKwInOut{Output}{Output}
    \Input{Low-fidelity environment $\Sigma^{\text{lo}}$, High-fidelity environment $\Sigma^{\text{hi}}$, discount $\gamma$, exploration noise $\epsilon$, $\mathcal{T}: \mathcal{S}^{\text{hi}} \rightarrow \mathcal{S}^{\text{lo}}$, number of low-fidelity trajectories from a state-action pair $m$.}
    \Output{$\hat{Q}_{*}^{\text{MFMC}}(s^{\text{hi}},a), \forall (s^{\text{hi}},a) \in \mathcal{S^\text{hi} \times A}, \pi^*_{\epsilon-\text{opt}}$}
    \textbf{Initialize:} $\pi \leftarrow $ an arbitrary $\epsilon$-soft policy, 
    $\hat{Q}^{\text{MFMC}}(s^{\text{hi}},a)=0, \forall (s^{\text{hi}},a) \in \mathcal{S^{\text{hi}} \times A},~~~~~~~~~~~~~~~~~~~$  $RetsH(s^{\text{hi}},a)$ $\leftarrow$ empty list, $RetsL(s^{\text{hi}},a)$ $\leftarrow$ empty list, $\forall (s^\text{hi},a) \in \mathcal{S^\text{hi} \times A}$, $RetsLP(s^{\text{lo}},a)$ $\leftarrow$ empty list, $\forall (s^\text{lo},a) \in \mathcal{S^\text{lo} \times A}.$ \\  

\For{Episode $=1, 2, \cdots, $}
{
Generate a trajectory $\tau^\text{hi}$ by following $\pi$ in $\Sigma^{\text{high}}$, $\tau^\text{hi}
: s_0^\text{hi},a_0,r^\text{hi}_{1},\cdots, s_{T-1}^\text{hi},a_{T-1},s^{\text{hi}}_{T}$.\\

Evaluate low-fidelity reward function $\mathcal{R}^{\text{lo}}(\mathcal{T}(s^\text{hi}),a), \forall (s^\text{hi},a) \in \tau^\text{hi}$ to generate $\tau^\text{lo}
: s_0^\text{lo},a_0,r^\text{lo}_{1},\cdots, s_{T-1}^\text{lo},a_{T-1},s^{\text{lo}}_{T}$.

Collect $m$ additional trajectories for every $(s^\text{lo},a) \in \tau^\text{lo} $ by executing $\pi$\footnote{Because $\mathcal{S}^\text{lo} \subseteq \mathcal{S}^\text{hi}$, executing $\pi$ in $\Sigma^\text{lo}$ amounts to executing the $\epsilon$-soft policy derived based on $\hat{Q}^{\text{lo}}(s^\text{lo},a) = \sum_{s^\text{hi}: \mathcal{T}(s^\text{hi})=s^\text{lo}} \hat{Q}^{\text{MFMC}}(s^\text{hi},a)$.} in $\Sigma^\text{lo}$, and append the $m$ corresponding low-fidelity return samples to $RetsLP(s^{\text{lo}},a)$.


$\mathcal{G}^\text{hi} \leftarrow 0, \mathcal{G}^\text{lo} \leftarrow 0$

\For{$t=T-1,T-2,\cdots, 0$}
{

$\mathcal{G}^\text{hi} \leftarrow \gamma \mathcal{G}^\text{hi} + r^\text{hi}_{t+1}$\\
$\mathcal{G}^\text{lo} \leftarrow \gamma \mathcal{G}^\text{lo} + r^\text{lo}_{t+1}$\\

\If {$(s_t^\text{hi},a_t) \not\in s_0^\text{hi},a_0,\cdots,s_{t-1}^\text{hi},a_{t-1}$} 
{
\texttt{Append} $\mathcal{G}^\text{hi}$ to $RetsH(s_t^\text{hi},a_t)$\\
\texttt{Append} $\mathcal{G}^\text{lo}$ to $RetsL(s_t^\text{hi},a_t)$\\

$\mathbb{E}[\mathcal{G}^\text{hi}] \leftarrow \texttt{mean}[RetsH(s_t^\text{hi},a_t)]$\\
$\sigma^2[\mathcal{G}^\text{hi}] \leftarrow \texttt{var}[RetsH(s_t^\text{hi},a_t)]$\\
$\mathbb{E}[\mathcal{G}^\text{lo}] \leftarrow \texttt{mean}[RetsL(\mathcal{T}(s_t^\text{hi}),a_t)]$\\
$\sigma^2[\mathcal{G}^\text{lo}] \leftarrow \texttt{var}[RetsL(\mathcal{T}(s_t^\text{hi}),a_t)]$\\
$\rho[\mathcal{G}^\text{hi},\mathcal{G}^\text{lo}] \leftarrow \frac{\texttt{cov}[RetsH(s_t^\text{hi},a_t),RetsL(s_t^\text{hi},a_t)]}{\sqrt{\sigma^2[\mathcal{G}^\text{hi}]\sigma^2[\mathcal{G}^\text{lo}]}} $\\
$\hat{Q}^{\text{MFMC}}(s_t^\text{hi},a_t)\leftarrow\mathbb{E}[\mathcal{G}^\text{hi}] +\rho[\mathcal{G}^\text{hi},\mathcal{G}^\text{lo}]\sqrt{\frac{\sigma^2[\mathcal{G}^\text{hi}]}{\sigma^2[\mathcal{G}^\text{lo}]}}\Big(
\texttt{mean}[RetsLP(\mathcal{T}(s_t^\text{hi}),a)]-\mathbb{E}[\mathcal{G}^\text{lo}] \Big) $\\
$a^* \leftarrow \underset{a}{\text{argmax}}~ \hat{Q}^{\text{MFMC}}(s_t^\text{hi},a)$\\
$\forall a \in \mathcal{A}(s_t^\text{hi}):$ $\pi(a|s_t^\text{hi}) \leftarrow \begin{cases} 1-\epsilon+\epsilon/|\mathcal{A}(s_t^\text{hi})|~~~~\text{if}~~~~ a= a^* \\ \epsilon/|\mathcal{A}(s_t^\text{hi})|~~~\text{if}~~~~ a\neq a^* \end{cases}$
}
}
}

    \caption{\texttt{MFMCRL}: Multifidelity Monte Carlo RL} \label{alg1:MCMFRL}
\end{algorithm}

\section{Proofs}

\subsection{Proof of lemma 1}

\paragraph{Statement.} \textit{Let $X_1, X_2, \cdots, X_n$ be i.i.d. r.vs. with mean $\mathbb{E}[X_i]=\mu_{_X}$, and variance $\mathbb{E}[(X_i-\mu_{_X})^2]=\sigma_{_X}^2$, $\forall i\in [n]$. Furthermore, suppose that $X_i, \forall i$, are bounded almost surely with a parameter $b$, namely,  $\text{Pr}(|X_i - \mu_{_X}| \leq b)=1, \forall i$. Then 
\begin{equation} 
    \text{Pr}\Bigg(\left\lvert\frac{1}{n}\sum_{i=1}^n X_i - \mu_{_X} \right\lvert \geq \xi \Bigg) \leq 2\text{exp}\bigg(\frac{-n\xi^2}{4\sigma_{_X}^2} \bigg),
\end{equation}
for $0\leq \xi \leq \sigma_{_X}^2/b$.}

\begin{proof}
It is straightforward to show that r.vs. $X_i$ satisfy the Bernstein condition with parameter $b$:
\begin{equation} \label{eq:bern}
\nonumber
    |\mathbb{E}\big[(X_i-\mu_{_X})^k \big]| \leq \frac{1}{2} k! \sigma_{_X}^2 b^{k-2}, \forall k=3,4,\cdots.
\end{equation}
By applying a Chernoff bound and using Bernstein's condition, we obtain the following upper tail bound for the event $\frac{1}{n} \sum_{i=1}^n X_i - \mu_{_X}  \geq \xi$ given that $\lambda \in (0, \frac{n}{2b}]$:
\begin{equation}
\nonumber
\begin{aligned}
    \text{Pr}\Bigg(\frac{1}{n} \sum_{i=1}^n X_i - \mu_{_X}  \geq \xi \Bigg) &\leq e^{-\lambda\xi} \mathbb{E}\big[e^{\frac{\lambda}{n}(\sum_{i=1}^n X_i - \mu_{_X})}\big] 
    \\&= e^{-\lambda\xi} \mathbb{E}\big[\prod_{i=1}^n e^{\frac{\lambda}{n}( X_i - \mu_{_X})}\big] 
    \\& \overset{(i)}{=} e^{-\lambda\xi}\prod_{i=1}^n \mathbb{E}\big[ e^{\frac{\lambda}{n}( X_i - \mu_{_X})}\big]
    \\&\overset{(ii)}{\leq} \text{exp}\big(-\lambda \xi + \lambda^2 \frac{\sigma_{_X}^2}{n}\big), 
\end{aligned}
\end{equation}
where $(i)$ follows by independence of the r.vs. and $(ii)$ follows from using Bernstein's condition in the Taylor expansion of $\mathbb{E}\big[ e^{\frac{\lambda}{n}( X_i - \mu_{_X})}\big]$ and noting that  $\lambda \leq \frac{n}{2b}$:
\begin{equation}
\begin{aligned}
\mathbb{E}\big[ e^{\frac{\lambda}{n}( X_i - \mu_{_X})}\big] &= 1+\frac{\lambda^2}{2n^2}\sigma_X^2+\sum_{k=3}^\infty \frac{\lambda^k}{k!n^k}\mathbb{E}\big[ ( X_i - \mu_X)^k\big]
\\& \leq 1+\frac{\lambda^2 \sigma_X^2}{2n^2}+\frac{\lambda^2 \sigma_X^2}{2n^2}\sum_{k=3}^\infty \frac{\lambda^{k-2}b^{k-2}}{n^{k-2}}
\\& = 1+\frac{\lambda^2 \sigma_X^2}{2n^2}\bigg[\frac{1}{1-|\frac{\lambda b}{n}|} \bigg], ~\forall |\lambda| < \frac{n}{b}
\\& \leq 1+\frac{\lambda^2 \sigma_X^2}{n^2}, ~\forall |\lambda| \leq \frac{n}{2b}
\\& \leq \text{exp}\big(\frac{\lambda^2 \sigma_X^2}{n^2}\big).
\end{aligned}
\end{equation}
The tightest bound is then obtained by finding the $ \underset{\lambda \in (0,\frac{n}{2b}]}{\text{inf}}~\text{exp}\big(-\lambda \xi + \lambda^2 \frac{\sigma_{_X}^2}{n}\big)$. This is  attained at $\lambda^* = \frac{\xi n}{2\sigma_{_X}^2} \leq \frac{n}{2b}$, which yields the condition $\xi \leq \frac{\sigma^2_{_X}}{b}$. The factor of 2 in \eqref{eq:chernof} captures the two-sided event.
\end{proof}

\subsection{Corollary of Lemma 1}

The concentration bound of Lemma 1 can be used to derive the minimum sample size that is required to ensure that the sample average deviates by no more than $\xi$ from the true mean with high probability, as stated in Corollary 1.
\begin{corollary}
With probability at least $1-\delta$, the difference between the sample mean $\frac{1}{n}\sum_{i=1}^n X_i$ and the population mean $\mu_{_X}$ is at most $\xi$ if $n \geq \frac{4 \sigma_{_X}^2}{\xi^2} \text{log}(\frac{2}{\delta})$ and $\xi \leq \sigma_{_X}^2/b$.
\end{corollary}

\begin{proof}
Set $n \geq \frac{4 \sigma_{_X}^2}{\xi^2} \text{log}(\frac{2}{\delta})$, and apply the concentration bound of Lemma 1 to find the probability of the complementary event $\text{Pr}\big(\left\lvert\frac{1}{n}\sum_{i=1}^n X_i - \mu_{_X} \right\lvert \leq \xi \big)$. $\blacksquare$
\end{proof}

The result of Corollary 1 relates the minimum sample size to the desired accuracy parameters $\xi, \delta$, as well as to the variance of the r.v. One can  observe that the smaller the variance of a r.v. is, the smaller the minimum required sample size can be. Put differently, for a given sample size and a significance parameter $\delta$, the maximum deviation error $\xi$ will be smaller for the r.v. with a smaller variance. 

\subsection{Proof of Theorem 1}

\paragraph{Statement.} \textit{To guarantee that
\begin{enumerate}[leftmargin=0.6cm]
    \item $\text{Pr}\Big(|\hat{Q}^{\text{hi}}_{\pi,n}(s^{\text{hi}},a) - {Q}^{\text{hi}}_{\pi}(s^{\text{hi}},a) | \leq \xi \Big) \geq 1-\delta $, then $n\geq \frac{4 \sigma_{\text{hi}}^2{(s^{\text{hi}},a)}}{\xi^2} \text{log}(\frac{2}{\delta})$.
    \item $\text{Pr}\Big(|\hat{Q}^{\text{MFMC}}_{\pi,n}(s,a) - {Q}^{\text{hi}}_{\pi}(s^{\text{hi}},a) | \leq \xi \Big) \geq 1-\delta$, then $n\geq \frac{4(1-\rho_{s,a}^2)\sigma_{\text{hi}}^2{(s^{\text{hi}},a)}}{\xi^2} \text{log}(\frac{2}{\delta})$.
\end{enumerate}
}

\begin{proof}
The single high-fidelity estimator $\hat{Q}^{\text{hi}}_{\pi,n}(s^{\text{hi}},a)$ is a sample mean of $n$ $\mathcal{G}^{\text{hi}}_{i}$ r.vs., each with a variance of $\sigma^2_{\text{hi}}(s^{\text{hi}},a)$. Furthermore, $|\mathcal{G}^{\text{hi}}_{i}| \leq \frac{\text{max}\{R^{\text{hi}}_{\text{min}}, R^{\text{hi}}_{\text{max}} \}}{1-\gamma}$ almost surely, and hence Bernstein's condition is satisfied with parameter $b=\frac{\text{max}\{R^{\text{hi}}_{\text{min}}, R^{\text{hi}}_{\text{max}} \}}{1-\gamma}$. By a straightforward application of Corollary 1, the first statement of the theorem is proved. On the other hand, the multifidelity estimator $\hat{Q}^{\text{MFMC}}_{\pi,n}(s^{\text{hi}},a)$ is a sample mean of $n$ r.vs,  $Y_i=\mathcal{G}^{\text{hi}}_{i} + \alpha^*_{s,a}Q_\pi^{\text{lo}}(\mathcal{T}(s^{\text{hi}}),a)- \alpha^*_{s,a}\mathcal{G}^{\text{lo}}_{i}$. It is 
straightforward to see that there exists a parameter $b^\prime$ such that $|Y_i| \leq b^\prime$ and that Bernstein's condition  is also satisfied. Lastly, $\text{Var}[Y_i] =\text{Var}[\mathcal{G}^{\text{hi}}_{i}]+(\alpha^*_{s,a})^2\text{Var}[\mathcal{G}^{\text{lo}}_{i}] -2\alpha_{s,a}^*\text{Cov}[\mathcal{G}^{\text{hi}}_{i},\mathcal{G}^{\text{lo}}_{i}]=(1-\rho_{s,a}^2)\sigma^2_{\text{hi}}(s^{\text{hi}},a)$, and the second statement of the theorem again follows by the application of Corollary 1. $\blacksquare$
\end{proof}

\subsection{Proof of Theorem 2}

\paragraph{Statement.}
\textit{Suppose that the number of trajectories from a state-action pair at a target state $s^{\text{hi}} \in \mathcal{S}^{\text{hi}}$ is the same for all actions $a \in \mathcal{A}$ and that $a_1$ is the greedy action with respect to the true $Q^\text{hi}_\pi(s^{\text{hi}},a)$. Furthermore, suppose that 
$\mathcal{P}^{\text{hi}}(s^{\text{hi}^\prime}|s^{\text{hi}},a)\geq \beta(s^{\text{hi}}), \forall s^{\text{hi}^\prime} \in \mathcal{S}^{\text{hi}}$. Then
\begin{enumerate}[leftmargin=0.6cm]
    \item $\text{Pr}\big(a_1= \underset{{a \in \mathcal{A}}}{\text{argmax}}~\hat{Q}^{\text{hi}}_{\pi,n}(s^{\text{hi}},a) \big) \geq \prod_{i=2}^{|\mathcal{A}|} \frac{\Delta^2_i}{\Delta^2_i+\text{Var}[\hat{Q}^{\text{hi}}_{\pi,n}(s^{\text{hi}},a_1)]+\text{Var}[\hat{Q}^{\text{hi}}_{\pi,n}(s^{\text{hi}},a_i)]}$.
    \item $\text{Pr}\big(a_1= \underset{{a \in \mathcal{A}}}{\text{argmax}}~\hat{Q}^{\text{MFMC}}_{\pi,n}(s^{\text{hi}},a) \big) \geq \prod_{i=2}^{|\mathcal{A}|} \frac{\Delta^2_i}{\Delta^2_i+(1-\rho^2_{s,a_1})\text{Var}[\hat{Q}^{\text{hi}}_{\pi,n}(s^{\text{hi}},a_i)]+(1-\rho^2_{s,a_i})\text{Var}[\hat{Q}^{\text{hi}}_{\pi,n}(s^{\text{hi}},a_i)]}$.
\end{enumerate}
}

\begin{proof}
First, consider the single high-fidelity estimator $\hat{Q}^{\text{hi}}_{\pi,n}(s^{\text{hi}},a)$, and let r.v. $B^\text{hi}_i= \hat{Q}^{\text{hi}}_{\pi,n}(s^{\text{hi}},a_1) - \hat{Q}^{\text{hi}}_{\pi,n}(s^{\text{hi}},a_i)$. Then $\text{Pr}\big(a_1= \underset{{a \in \mathcal{A}}}{\text{argmax}}~\hat{Q}^{\text{hi}}_{\pi,n}(s,a) \big)=\text{Pr}(\text{min}(B^\text{hi}_2, \cdots, B^\text{hi}_{|\mathcal{A}|}) \geq 0)=
\text{Pr}(\cap_{i=2}^{|\mathcal{A}|} \{ B^\text{hi}_i  \geq 0 \}) \overset{(i)}{\geq} \prod_{i=2}^{|\mathcal{A}|} \text{Pr}(  B^\text{hi}_i  \geq 0)$, where $(i)$ follows because the events $\{B_i^{hi}\geq 0\}, \forall i$, are positively correlated.
In addition, 
\begin{equation}
\begin{aligned}
  \text{Pr}(B^\text{hi}_i\geq 0) &= \text{Pr}(B^\text{hi}_i - \mathbb{E}[B^\text{hi}_i]\geq -\Delta_i)\\
  &\overset{(ii)}{\geq} \frac{\Delta_i^2}{\Delta_i^2 + \text{Var}[B^\text{hi}_i]}\\
  &\overset{(iii)}{\geq} \frac{\Delta^2_i}{\Delta^2_i+\text{Var}[\hat{Q}^{\text{hi}}_{\pi,n}(s^{\text{hi}},a_1)]+\text{Var}[\hat{Q}^{\text{hi}}_{\pi,n}(s^{\text{hi}},a_i)]},
\end{aligned}
\end{equation}
where $(ii)$ follows by Cantelli's inequality  and  $(iii)$ follows as long as $\text{Cov}[\hat{Q}^{\text{hi}}_{\pi,n}(s^{\text{hi}},a_1),\hat{Q}^{\text{hi}}_{\pi,n}(s^{\text{hi}},a_i)] \geq 0, \forall i$. Notice that $\text{Cov}[\hat{Q}^{\text{hi}}_{\pi,n}(s^{\text{hi}},a_1),\hat{Q}^{\text{hi}}_{\pi,n}(s^{\text{hi}},a_i)]=$ $ =  \text{Cov}\big[\frac{1}{n}\sum_{k=1}^{n} \mathcal{G}^{\text{hi}}_{k, a_1},\frac{1}{n}\sum_{k=1}^{n} \mathcal{G}^{\text{hi}}_{k,a_i} \big]= \frac{1}{n^2} \sum_{k}\sum_{l}\text{Cov}\big[ \mathcal{G}^{\text{hi}}_{k, a_1},\mathcal{G}^{\text{hi}}_{l,a_i} \big]$, where $\mathcal{G}^{\text{hi}}_{k, a_1}$ and $\mathcal{G}^{\text{hi}}_{k, a_1}$ are sample returns from the pairs $(s^{\text{hi}},a_1)$ and $(s^{\text{hi}},a_i)$, respectively. $\text{Cov}\big[ \mathcal{G}^{\text{hi}}_{k, a_1},\mathcal{G}^{\text{hi}}_{l,a_i} \big]$ is non-zero if the pairs $(s^{\text{hi}},a_1)$ and $(s^{\text{hi}},a_i)$ show up in the same trajectory; otherwise the two samples are independent, and the covariance is zero. If $(s^{\text{hi}},a_1)$ and $(s^{\text{hi}},a_i)$ do show up in the same trajectory, then they are non-negatively correlated as  $\mathcal{P}^{\text{hi}}(s^{\text{hi}}|s^{\text{hi}^\prime},a)\geq \beta(s^{\text{hi}}), \forall s^{\text{hi}} \in \mathcal{S}^{\text{hi}}$.\footnote{When $(s^{\text{hi}},a_1)$ and $(s^{\text{hi}},a_i)$ show up in the same trajectory, the trajectory can be broken into two pieces, one that starts at $(s^{\text{hi}},a_1)$ and ends at $(s^{\text{hi}},a_i)$ and another that starts $(s^{\text{hi}},a_i)$ and ends at $s^{\text{hi}}_T$. The returns of $(s^{\text{hi}},a_1)$ are the discounted sum of all rewards in the two pieces, whereas the returns of $(s^{\text{hi}},a_i)$ are the discounted sum of the rewards in the second piece only. The condition $\mathcal{P}^{\text{hi}}(s^{\text{hi}}|s^{\text{hi}^\prime},a)\geq \beta(s^{\text{hi}}), \forall s^{\text{hi}} \in \mathcal{S}^{\text{hi}}$ means the second piece is more likely to happen if the first piece happens and vice versa, which induces the non-negative correlation among the two return samples.} Hence, $\text{Cov}\big[ \mathcal{G}^{\text{hi}}_{k, a_1},\mathcal{G}^{\text{hi}}_{l,a_i} \big] \geq 0$. By following a similar argument, we can show the second part of the theorem.
\end{proof}

\section{Experimental details}

\subsection{Computing infrastructure}
Our computational experiments and implementation of \texttt{MFMCRL} are  based on Python $3.8.11$ and NumPy $1.19.5$. Our experiments have been conducted on a $36$-core machine running CentOS Linux $7$ (Core) with an Intel Xeon E5-2695v4 CPU. We use \texttt{mpi4py} $3.1.2$ to run experiments with different random seeds in parallel. We use the API provided by \texttt{NASLib}\footnote{https://github.com/automl/NASLib} $0.1.0$ to retrieve the validation accuracy curves of all candidate architectures in NAS-Bench-201 and store them in a python dictionary to speed up reward retrieval for the RL agent. Data dictionaries and all our codes are available in the supplementary material. 

\subsection{Synthetic multifidelity MDPs}

We first present the random process that is used to generate the high-fidelity transition function $\mathcal{P}^{\text{hi}}$. The transition probability from a state $s^\text{hi}$ and action $a$  to  successor states $ s^{\text{hi}^\prime} \in \{0,1,\cdots,|\mathcal{S}^{\text{hi}}|-1\}$ is sampled from $ \mathcal{U}_{|\mathcal{S}^{\text{hi}}|-1} \cdot \mathbbm{1}$, where $\mathcal{U}_{|\mathcal{S}^{\text{hi}}|-1}$ is a uniform random vector of size $|\mathcal{S}^{\text{hi}}|-1$ defined over the interval $[0,1]$, $\mathbbm{1}$ is a random binary vector of size $|\mathcal{S}^{\text{hi}}|-1$ whose $i$-th element is 1 if $\mathcal{U}^i_{1} >\mathcal{U}^{\text{ref}}_1 $, and $\cdot$ denotes element-wise multiplication. Here, $\mathcal{U}^i_{1}, \mathcal{U}^{\text{ref}}_1 \sim \mathcal{U}[0,1]$, and  $\mathcal{U}_{\text{ref}}$ is sampled once per $(s^\text{hi},a)$. The sampled vector is normalized such that $\sum_{s^{\text{hi}^\prime} \in \{0,\cdots, |\mathcal{S}^{\text{hi}}|-1\}} \mathcal{P}^\text{hi}(s^{\text{hi}^\prime}|s^\text{hi},a) = 1-p_t$, where $p_t$ is the transition probability from any state $s^{\text{hi}^\prime} \in \{0,\cdots, |\mathcal{S}^{\text{hi}}|-1\}$ to the terminal state $s^{\text{hi}^\prime} \in \{|\mathcal{S}^{\text{hi}}|\}$. In our experiments we set $p_t = 0.1$. $s^{\text{hi}^\prime} \in \{|\mathcal{S}^{\text{hi}}|\}$ is an absorbing state and so the transition probability from this state to any other state is 0 regardless of the action, and the transition probability to itself is 1. The high-fidelity reward function at a state $s^{\text{hi}^\prime} \in \{0,1,\cdots,|\mathcal{S}^{\text{hi}}|-1\}$ and action $a$, $\mathcal{R}^\text{hi}(s^{\text{hi}},a)$ is sampled from $\mathcal{U}_1 \cdot \mathbbm{1}[i]$, where $\mathbbm{1}[i]$ is an element chosen uniformly at random from the random binary vector $\mathbbm{1}$ described earlier. For the terminal state, $s^{\text{hi}} \in \{|\mathcal{S}^{\text{hi}}|\}$, $\mathcal{R}^\text{hi}(s^{\text{hi}},a)=0$.  The random generation process of $\mathcal{P}^{\text{hi}}$ and $\mathcal{R}^{\text{hi}}$ is an adaptation of that in the $\texttt{pymdptoolbox}$ and is implemented in the $\texttt{synthetic\_mdp\_envs}/\texttt{MDPGen}$ class in our code. In Figure \ref{fig:Perf2}, we provide results similar to those of Figure \ref{fig:Perf} in the main text but for a random MDP with $|\mathcal{S}|=1000,  |\mathcal{A}|=12$.

\begin{figure}[h] 
\centering
\subfloat[$\Sigma^\text{lo}$ has an SNR of -3dB]{\includegraphics[width=0.33\columnwidth]{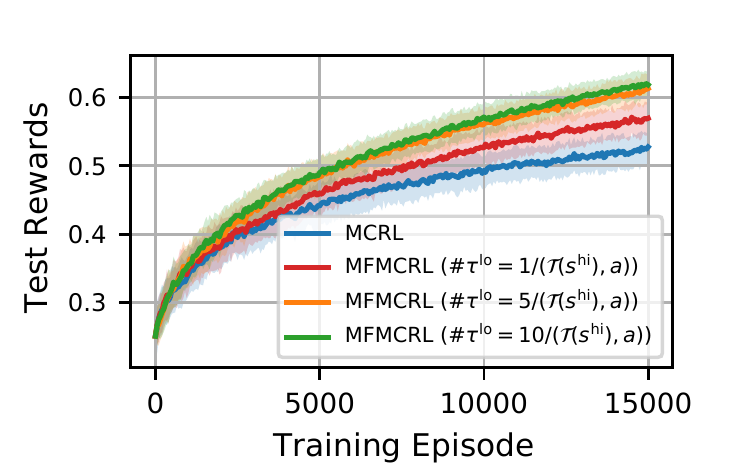}}
\hfill
\centering
\subfloat[$\#\tau^\mathrm{lo} = 10/(\mathcal{T}(s^\mathrm{hi}),a))$]{\includegraphics[width=0.33\columnwidth]{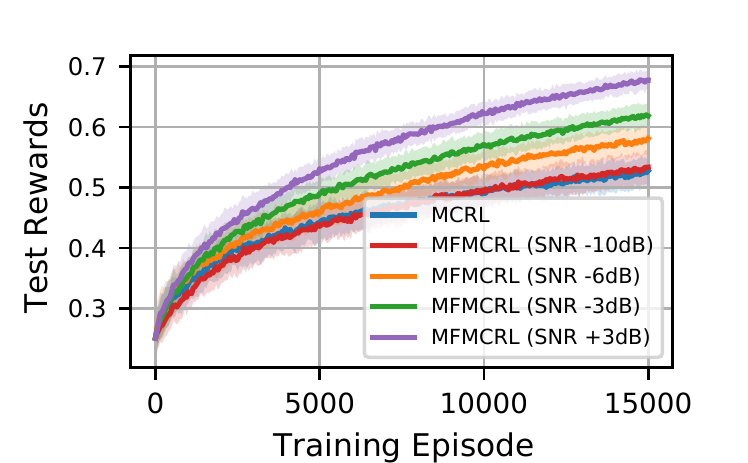}} 
\hfill
\centering
\subfloat[Variance reduction factor]{\includegraphics[width=0.33\columnwidth]{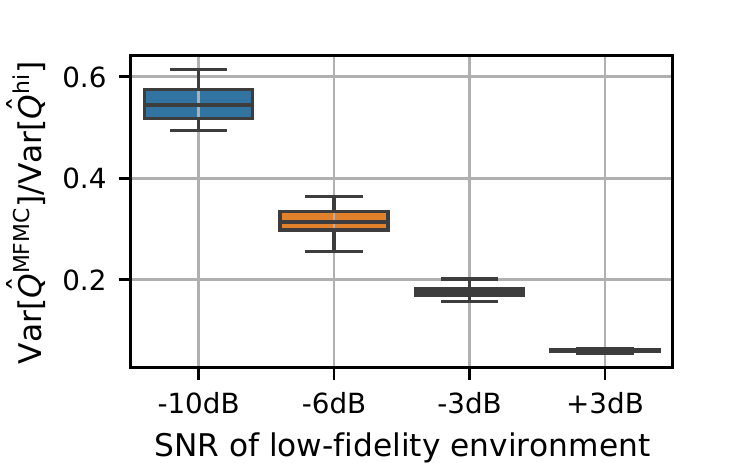}} 
\hfill
\caption{Mean and standard deviation of test episode rewards for the proposed \texttt{MFMCRL} during training: (a) test episode rewards improve with increasing number of low-fidelity samples ($\#\tau^\text{lo}$); (b) test episode rewards improve with less noisy low-fidelity environments; (c) variance reduction factor improves when low- and high-fidelity environments are more correlated. These results are based on a random MDP with $|\mathcal{S}|=1000, |\mathcal{A}|=12$. 
}
\label{fig:Perf2}
\end{figure}

\clearpage

\subsection{NAS}
Neural architectures can be represented as a Directed Acyclic Graph (DAG) $\mathcal{G}$ which describes the operations and order of operations that are used to process the data in a deep learning model. The performance of an architecture $\mathcal{G}$ is  $f(\mathcal{G}, \mathcal{D}, L): (\mathcal{G}, \mathcal{D}, L) \rightarrow \mathbb{R}$, where $\mathcal{D}$ is the training dataset of a given machine learning (ML) task, and $L$ is the number of training epochs. The choice of $f(\mathcal{G}, \mathcal{D}, L)$ depends on the type of ML task and the design objectives, but is usually a metric evaluated on a held-out validation dataset. 

In NAS-Bench-201 \cite{dong2020bench}, the search space is defined over the inner structure of a convolutional cell, which is then stacked to form a classification model that is trained on three datasets, $\mathcal{D} = \{ \texttt{CIFAR-10}, \texttt{CIFAR-100},  \texttt{ImageNet16-120} \}$. The DAG of a cell in NAS-Bench-201 is made of $4$ vertices and therefore has $6$ possible edges. Each edge can assume one of the following choices: $\{$\texttt{zeroize,skip-connect,1x1 conv,3x3 conv,3x3 avg pool}$\}$, where \texttt{zeroize} is the operation of dropping the edge. An architecture $\mathcal{G}$ in NAS-Bench-201 can be therefore represented by a 6-dimensional vector $\mathbf{E}=[e_0,e_1,e_2,e_3,e_4,e_5]$, where each element specifies the edge value from one of the aforementioned 5 operations. Based on this search space, there are $5^6=15,625$ unique architectures. NAS-Bench-201 provides the complete training and validation accuracy curves of each architecture, trained independently over the three aforementioned datasets and for $L=200$ training epochs.

We first discuss the general formulation of NAS as an RL problem, and then discuss the construction of multifidelity environments. In the NAS-RL environment, episodes are started from an architecture based on the initial state distribution. Based on the current architecture (state), the agent chooses a new value (action) for a randomly selected edge to create a new architecture (new state). The new architecture is then evaluated on a held-out validation dataset, and the validation accuracy is provided to the agent as a reward. In NAS-Bench-201, evaluating the validation accuracy of an architecture is a simple table lookup. Below, we provide the detail description of the formulation, 
\begin{enumerate}
    \item \textbf{State space}: the state space is the set of all possible  architectures $\mathbf{E}=[e_0,e_1,e_2,e_3,e_4,e_5]$, in addition to a special state variable  $ \mathbbm{I} \in \{0,1,\cdots,5,6\}$ that determines which edge will be configured/edited by the agent ($ \mathbbm{I} \in \{0,1,\cdots,5\}$), or whether the episode should be terminated ($ \mathbbm{I} = 6$), in which case the episode is restarted from another state based on the initial state distribution. Hence, $\mathcal{S} = [\mathbf{E},\mathbbm{I}]$, and $|\mathcal{S}| = 15,625 \times 7 = 109,375$.
    
    \item \textbf{Action space}: the action space is the set of all possible edge values $\mathcal{A}=\{0,1,2,3,4\}$.
    \item \textbf{Reward}: the reward of a state-action pair is the validation accuracy of the new architecture. The new architecture is identical to the current architecture, except for 
    edge $\mathbb{I} \in \{0,\cdots,5\}$ which is assigned a new operation based on the agent's action $a \in \mathcal{A}$. If $\mathbb{I}=6$, the reward is 0 regardless of the action because it is a terminal state.
    \item \textbf{Transition dynamics}: the successor state of a state-action pair is the same as the current state except for two state variables. First, one of the edges will assume a new value based on the current action, $\mathbf{E}[\mathbbm{I}]=a$. Second, the new $\mathbbm{I}$ is chosen uniformly at random from $ \{0,1,\cdots,5,6\}$.
    \item \textbf{Initial state distribution}: the initial $\mathbf{E}$ is a baseline architecture given by $[\texttt{3x3 conv}, \texttt{3x3 avg pool}, \texttt{3x3 conv}, \texttt{3x3 avg pool}, \texttt{3x3 conv}, \texttt{3x3 avg pool}]$. The initial  $\mathbbm{I}$ is chosen uniformly at random from $ \{0,1,\cdots,5\}$.
\end{enumerate}

Based on the above formulation, we construct two multifidelity scenarios as follows. In both scenarios, the validation accuracy of an architecture at the end of training (i.e., at $L=200$ epochs) is used as a high-fidelity reward in the high-fidelity environment. For the low-fidelity environment, we have two cases: 

\begin{enumerate}
    \item Case (i): low-fidelity environment is identical to the high-fidelity environment except for the reward function, which is now the validation accuracy at the $L=10$th training epoch. 
    \item Case (ii): low-fidelity environment is defined for a smaller search space and the reward function is the validation accuracy of an architecture at the $L=10$th training epoch. In this case, $e_0$ is fixed to the \texttt{zeroize} operation (i.e. the first edge is dropped). Hence,  $|\mathcal{S}^\text{lo}|= 5^5 * 6 = 18,750$, and $\mathcal{S}^\text{lo} \subset \mathcal{S}^\text{hi}$. A high-fidelity state $s^\text{hi}$ can be mapped into a low-fidelity state $s^\text{lo}$ by setting $s^\text{lo} = [1,e_1^\text{hi},e_2^\text{hi},e_3^\text{hi},e_4^\text{hi},e_5^\text{hi}, \mathbbm{I}^\text{hi} - 1]$.
\end{enumerate}

In Figure \ref{fig:nas2}, we provide results which are similar to those on the $\texttt{ImageNet16-120}$ dataset in Figure \ref{fig:nas} of the main text but for the two other datasets in NAS-Bench-201, $\texttt{CIFAR-10}$ and $ \texttt{CIFAR-100}$.

\begin{figure}[h] 
\centering
\subfloat[$\texttt{CIFAR-10}$]{\includegraphics[width=0.5\columnwidth]{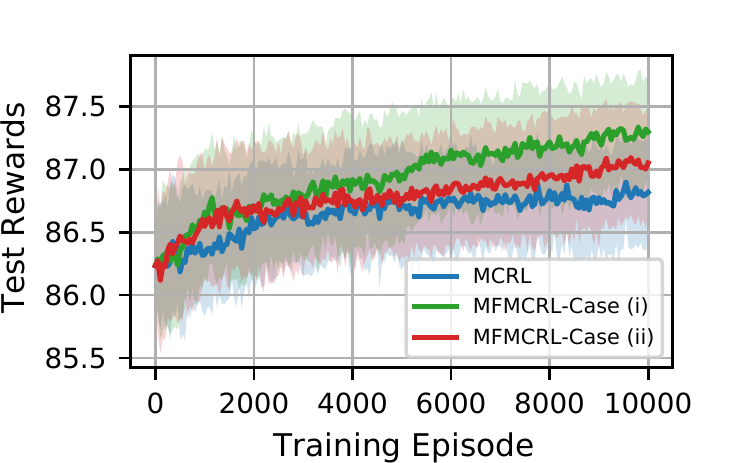}}
\hfill
\centering
\subfloat[$\texttt{CIFAR-100}$]{\includegraphics[width=0.5\columnwidth]{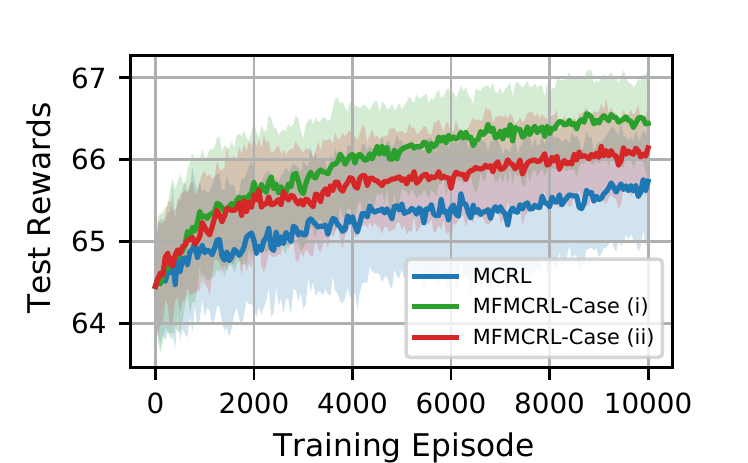}} 
\hfill
\caption{Mean and standard deviation of test episode rewards for the proposed \texttt{MFMCRL} during training on multifidelity NAS environments. The two cases (i) and (ii) are described in the text. In both cases, $\#\tau^\mathrm{lo} = 5/(\mathcal{T}(s^\mathrm{hi}),a))$. }
\label{fig:nas2}
\end{figure}

\end{document}